\documentclass{article}

\usepackage{microtype}
\usepackage{graphicx}
\usepackage{subfigure}
\usepackage{booktabs} %

\usepackage{hyperref}
\usepackage[authoryear]{natbib}
\setcitestyle{round}

\renewcommand{\cite}{\citep}

\usepackage{amsmath}
\usepackage{amssymb}
\usepackage{mathtools}
\usepackage{amsthm}

\usepackage[capitalize,noabbrev]{cleveref}

\usepackage[textsize=tiny]{todonotes}

\usepackage{mathtools}
\usepackage{booktabs}
\usepackage{amsthm,bbm}
\usepackage{cleveref} %
\usepackage{subcaption} %
\usepackage{mathtools}
\usepackage{amsmath}
\usepackage{stmaryrd}
\usepackage{etoolbox} %
\RequirePackage{paralist}
\usepackage{enumitem}
\usepackage{multirow} %
\usepackage{array}
\usepackage{graphicx}
\usepackage{etoc}
\usepackage{listings}
\usepackage{eqparbox}
\usepackage{bbm}
\usepackage{quoting}

\AtEndPreamble{%
  \newtheorem{remark}[theorem]{Remark}
  \newtheorem{assumption}{Assumption}
  \newtheorem{condition}{Condition}
  \newtheorem{property}{Property}
  \theoremstyle{plain}
  \newenvironment{definition*}
                 {\pushQED{\qed}\definition}
                 {\popQED\enddefinition}  
  \newenvironment{example*}
                 {\pushQED{\qed}\example}
                 {\popQED\endremark}  
  \newenvironment{remark*}
                 {\pushQED{\qed}\remark}
                 {\popQED\endremark}  
  \newenvironment{assumption*}
                 {\pushQED{\qed}\assumption}
                 {\popQED\endremark}  
  \newenvironment{property*}
                 {\pushQED{\qed}\property}
                 {\popQED\endremark}  
}

\crefname{assumption}{Assumption}{Assumptions}
\crefname{figure}{Fig{.}}{Figs{.}}%
\crefname{table}{Table}{Tables}
\crefname{definition}{Definition}{Definitions}
\crefname{theorem}{Theorem}{Theorems}
\crefname{lemma}{Lemma}{Lemmas}
\crefname{proposition}{Proposition}{Propositions}
\crefname{corollary}{Corollary}{Corollaries}
\crefname{problem}{Problem}{Problems}
\crefname{example}{Example}{Examples}
\crefname{fact}{Fact}{Facts}
\crefname{conjecture}{Conjecture}{Conjectures}
\crefname{remark}{Remark}{Remarks}
\crefname{condition}{Condition}{Conditions}
\crefname{requirement}{Requirement}{Requirements}
\crefname{enumi}{}{}
\crefname{equation}{Eq{.}}{Eqs{.}}
\crefname{section}{Section}{Sections}

\usepackage{amsmath,amsfonts,bm}

\def\eqref#1{equation~\ref{#1}}

\def\1{\bm{1}}

\DeclareMathAlphabet{\mathsfit}{\encodingdefault}{\sfdefault}{m}{sl}
\SetMathAlphabet{\mathsfit}{bold}{\encodingdefault}{\sfdefault}{bx}{n}

\DeclareMathOperator*{\argmax}{arg\,max}
\DeclareMathOperator*{\argmin}{arg\,min}

\newcommand{\mathboldcommand}[1]{\mathbb{#1}}

\newcommand{\bbF}{\mathboldcommand{F}}

\newcommand{\bbN}{\mathboldcommand{N}}

\newcommand{\bbR}{\mathboldcommand{R}}

\newcommand{\bbZ}{\mathboldcommand{Z}}

\newcommand{\bfa}{\mathbf{a}}
\newcommand{\bfb}{\mathbf{b}}
\newcommand{\bfc}{\mathbf{c}}

\newcommand{\bfm}{\mathbf{m}}

\newcommand{\bfp}{\mathbf{p}}

\newcommand{\bfs}{\mathbf{s}}

\newcommand{\bfv}{\mathbf{v}}
\newcommand{\bfw}{\mathbf{w}}
\newcommand{\bfx}{\mathbf{x}}
\newcommand{\bfy}{\mathbf{y}}
\newcommand{\bfz}{\mathbf{z}}

\usepackage{euscript}
\newcommand{\mathcalcommand}[1]{\mathcal{#1}}
\newcommand{\mcA}{\mathcalcommand{A}}
\newcommand{\mcB}{\mathcalcommand{B}}
\newcommand{\mcC}{\mathcalcommand{C}}

\newcommand{\mcS}{\mathcalcommand{S}}
\newcommand{\mcT}{\mathcalcommand{T}}

\newcommand{\mcX}{\mathcalcommand{X}}
\newcommand{\mcY}{\mathcalcommand{Y}}

\DeclareMathAlphabet{\mathpzc}{T1}{pzc}{m}{it}

\allowdisplaybreaks

\newtheorem{theorem}{Theorem}%

\newtheorem{lemma}[theorem]{Lemma}

\newtheorem{definition}{Definition}

\newcommand*{\commentout}[1]{}

\newlength{\parskiptrue}
\definecolor{lred}{rgb}{1.0, 0.5, 0.5}
\definecolor{lorange}{rgb}{1.00, 0.90, 0.20}
\definecolor{lgreen}{rgb}{0.35, 0.95, 0.35}
\definecolor{lime}{rgb}{0.9, 1.0, 0.6}
\definecolor{lblue}{rgb}{1.0, 0.85, 0.75}

\makeatletter

\newcommand*\wthelper[2]{%
        \hbox{\dimen@\accentfontxheight#1%
                \accentfontxheight#11.1\dimen@
                $\m@th#1\widetilde{#2}$%
                \accentfontxheight#1\dimen@
        }%
}

\newcommand*\whhelper[2]{%
        \hbox{\dimen@\accentfontxheight#1%
                \accentfontxheight#11.2\dimen@
                $\m@th#1\widehat{#2}$%
                \accentfontxheight#1\dimen@
        }%
}
\makeatother

\makeatletter
\newcommand{\oset}[3][0ex]{%
  \mathrel{\mathop{#3}\limits^{
    \vbox to#1{\kern-3\ex@
    \hbox{$\scriptstyle#2$}\vss}}}}
\makeatother

\newcommand*{\defeq}{\coloneq}

\newcommand*{\relu}{\mathrm{ReLU}}

\newcommand*{\fpq}{\mathbb{F}}
\newcommand*{\efpq}{{\overline{\mathbb{F}}}}
\newcommand*{\fmin}{\omega}
\newcommand*{\fmax}{\Omega}

\newcommand*{\round}[1]{ \left\lceil{#1}\right\rfloor }

\makeatletter

\newcommand{\xMapsto}[2][]{\ext@arrow 0599{\Mapstofill@}{#1}{#2}}
\def\Mapstofill@{\arrowfill@{\Mapstochar\Relbar}\Relbar\Rightarrow}
\makeatother
\makeatletter
\def\moverlay{\mathpalette\mov@rlay}
\def\mov@rlay#1#2{\leavevmode\vtop{%
   \baselineskip\z@skip \lineskiplimit-\maxdimen
   \ialign{\hfil$\m@th#1##$\hfil\cr#2\crcr}}}

\makeatother

\newcommand{\lrp}[1]{\left({#1}\right)}

\newcommand{\emin}{\mathfrak{e}_{\min}}
\newcommand{\emax}{\mathfrak {e}_{\max}}

\newcommand{\nan}{\text{\rm NaN}}

\newcommand{\mbit}{p}
\newcommand{\ebit}{q}

\newcommand*{\indcc}[1]{\mathbbm{1}_{#1}}

\newcommand*{\zerov}{\boldsymbol{0}}
\newcommand*{\onev}{\boldsymbol{1}}

\newcommand*{\din}{d_\text{in}}
\newcommand*{\dout}{d_\text{out}}

\usepackage{url}

\usepackage[left=3cm,right=3cm]{geometry}

\begin{document}
\title{On the Expressive Power of Floating-Point Transformers}
\date{ }

\author{ 
Sejun Park \thanks{Department of Artificial Intelligence, Korea University, Seoul, 02841 , Republic of Korea}
 \and
 Yeachan Park \thanks{Department of Mathematics and Statistics, Sejong University, Seoul, 05006, Republic of Korea }
 \and
 Geonho Hwang \thanks{Department of Mathematical Sciences, Gwangju Institute of Science and
Technology, Gwangju, 61005, Republic of Korea
\newline
Email: \texttt{sejun.park000@gmail.com, ychpark@sejong.ac.kr, hgh2134@gist.ac.kr } }    \thanks{Corresponding author}
}

\maketitle

\vskip 0.3in

\begin{abstract}
The study on the expressive power of transformers shows that transformers are permutation equivariant, and they can approximate all permutation-equivariant continuous functions on a compact domain.
However, these results are derived under real parameters and exact operations, while real implementations on computers can only use a finite set of numbers and inexact machine operations with round-off errors.
In this work, we investigate the representability of floating-point transformers that use floating-point parameters and floating-point operations.
Unlike existing results under exact operations, we first show that floating-point transformers can represent a class of non-permutation-equivariant functions even without positional encoding.
Furthermore, we prove that floating-point transformers can represent all permutation-equivariant functions when the sequence length is bounded, but they cannot when the sequence length is large.
We also found the minimal equivariance structure in floating-point transformers, and show that all non-trivial additive positional encoding can harm the representability of floating-point transformers.
\end{abstract}
\section{Introduction}\label{sec:intro}
A self-attention-based transformer is one of the most popular neural network architectures for sequence-to-sequence learning. 
Starting with natural language processing \cite{vaswani2017attention}, they have been showing state-of-the-art performances in various fields, including computer vision \cite{dosovitskiy2021an},  %
graph representation learning \cite{ying2021transformers}, and time-series forecasting \cite{zhou2021informer}.
Due to the symmetry in the transformer architectures, it is well-known that they are permutation equivariant: if input tokens are permuted, then the outputs are also permuted correspondingly.
While such a symmetry has been appreciated when the output is independent of the input order, e.g., \cite{lee2019set}, 
it is often intentionally broken by using positional encoding if the input order matters. %

Several works have investigated the expressive power of transformers. For example, it is known that transformers are dense in the $L^p$ space of permutation-equivariant sequence-to-sequence functions on a compact domain, where the permutation-equivariance condition can be removed with learnable positional encoding \cite{Yun2020Are}.
Similar results have also been shown for continuous function space equipped with the uniform norm \cite{alberti2023sumformer}, shallow architectures and probabilistic setup \cite{furuya2025transformers}, sparse architectures \cite{yun2020n}, and approximating shift-equivariant functions \cite{takakura2023approximation}. 
Furthermore, transformers can fit a finite number of input--output sequence pairs \cite{mahdavi2024memorization,pmlr-v258-dana25a}, using a smaller number of parameters with a deeper architecture \cite{kim2023provable}.

However, existing theoretical results cannot be directly applied to real-world transformers operating in computers since they assume real parameters and exact arithmetic. On the other hand, in reality, transformers use a finite set of numbers and inexact machine operations such as floating-point arithmetic. 
Only a few recent works have investigated the expressive power of neural networks under floating-point arithmetic. 
\citet{park2024expressive} show that fully-connected floating-point networks using $\relu$ or the binary step function can represent all floating-point functions from floating-point vectors to floating-point numbers.
This result has also been extended to a general class of activation functions \cite{hwang2025floating} and interval arithmetic \cite{hwang2025interval}.
One notable observation here is that floating-point networks may not inherit important properties of neural networks analyzed under exact arithmetic. For example, floating-point networks using the identity activation function can represent all floating-point functions \cite{hwang2025floating}, which is impossible under exact arithmetic.

\subsection{Contributions}\label{sec:contribution}
In this work, we study the expressive power of floating-point transformers that use floating-point parameters and floating-point operations (e.g., addition and multiplication). Specifically, we investigate whether floating-point transformers inherit well-known properties of real transformers under exact arithmetic, by answering the following questions:\vspace{-0.1in}
\begin{itemize}[leftmargin=0.15in]
\item Are floating-point transformers permutation equivariant?\vspace{-0.05in}
\item Can floating-point transformers represent all floating-point sequence-to-sequence functions?
\end{itemize}
\vspace{-0.1in}
Surprisingly, we prove that the answers to both questions are \emph{no}.
Namely, real-world transformers under floating-point arithmetic may not satisfy natural properties of transformers analyzed under exact arithmetic.

To answer these questions, we first show that floating-point transformers can express a class of non-permutation-equivariant functions. In particular, by exploiting the non-associativity of floating-point addition (i.e., different orders of summations may result in different outputs), %
we show that an attention layer can identify the order of the third to last input tokens without positional encoding when the input tokens are distinct. 
Using this attention layer and the information about the order of inputs, we design floating-point transformers representing a class of non-permutation-equivariant functions (\cref{thm:diagonal-approx}).

We next show that floating-point transformers cannot represent some permutation-equivariant sequence-to-sequence functions when the sequence length $n$ is large (\cref{thm:counterexample}). To prove this, we exploit the property of the floating-point addition that repeated addition of the same number may converge to some finite number. %
We further show that such an %
impossibility
result only holds when the sequence length $n$ is large: floating-point transformers can represent all permutation-equivariant sequence-to-sequence functions when $n$ is small (\cref{thm:perm-equiv-approx}).
We note that real transformers do not show such a dependency in the sequence length \cite{Yun2020Are}. %
We also investigate inductive biases in floating-point transformers (\cref{thm:equiv,thm:eq-preserve}) and discuss the effect of positional encoding on the expressive power of floating-point transformers.

\subsection{Organization}
We introduce notations and problem setup in \cref{sec:preliminary} and formally state our main results in \cref{sec:main_result}. We then discuss inductive biases in floating-point transformers and effects of positional encoding in \cref{sec:discussion}.
The proofs of our main results are in \cref{sec:proof}. We conclude the paper in \cref{sec:conclusion}.
\section{Preliminaries}\label{sec:preliminary}
\subsection{Notations}
Throughout this paper, we often use lower case letters ($a,b,c,\dots$) to denote scalar values, bold lower case letters ($\bfa,\bfb,\bfc,\dots$) to denote vectors, upper case letters ($A,B,C,\dots$) to denote matrices, calibrated upper case letters ($\mcA,\mcB,\mcC,\dots$) to denote sets. 
We use $\bbN$ and $\bbR$ %
to denote the set of positive integers and the set of real numbers, %
respectively.
For $n\in\bbN$, we use $[n]\defeq\{1,2,\dots,n\}$.
For $a,b\in\bbR\cup\{\pm\infty\}$ and $\mcS\subset\bbR$, $(a,b)\defeq\{x\in\bbR:a<x<b\}$, $[a,b]\defeq\{x\in\bbR:a\le x\le b\}$, $(a,b)_{\mcS}\defeq(a,b)\cap\mcS$, and $[a,b]_{\mcS}\defeq[a,b]\cap\mcS$.
We also define half-open intervals $[a,b)$, $(a,b]$, $[a,b)_{\mcS}$, and $(a,b]_{\mcS}$ in a similar way.

For $n\in\bbN$ and $x_1,\dots,x_n\in\bbR$, we use $(x_1,\dots,x_n)\in\bbR^n$ to denote the vector whose $i$-th coordinate is $x_i$.
Throughout this paper, we treat all vectors as column vectors (i.e., $\bbR^n=\bbR^{n\times1}$).
For $n,m\in\bbN$ and $\bfx_1,\dots,\bfx_n\in\bbR^n$, we use $[\bfx_1,\dots,\bfx_n]\in\bbR^{n\times m}$ to denote the $n\times m$ matrix whose $i$-th column is $\bfx_i$.
For $n,m\in\bbN$, $\bfx\in\bbR^n$, and $M\in\bbR^{n\times m}$, we use $x_i$ to denote the $i$-th coordinate of $\bfx$, $M_i$ to denote the $i$-th row of $M$, and $M_{ij}$ to denote the element at the $i$-th row and the $j$-th column of $M$. 
We also use $M_{i:j}$ to denote the submatrix of $M$ induced by $M_i,M_{i+1},\dots,M_j$. When $i=1$ or $j=n$, we use $M_{:j}$ or $M_{i:}$ to denote $M_{i:j}$, respectively.
We use $\zerov_n$ and $\onev_n$ to represent $(0,\dots,0)\in\bbR^n$ and $(1,\dots,1)\in\bbR^n$, respectively.
Likewise, we use $\zerov_{n\times m}$ and $\onev_{n\times m}$ to denote the $n\times m$ matrix consisting of zeros and the $n\times m$ matrix consisting of ones, respectively.

For a set $\mcX$, $\rho:\mcX\to\mcX$, $\sigma:\mcX^n\to\mcX^n$, we often apply $\rho$ and $\sigma$ to a matrix $M=[\bfm_1,\dots,\bfm_m]\in\mcX^{n\times m}$ as follows: $\rho(M),\sigma(M)\in\mcX^{n\times m}$ are matrices satisfying
\begin{align}
\rho(M)_{ij}=\rho(M_{ij}),~~
\sigma(M)=\big[\sigma(\bfm_1),\dots,\sigma(\bfm_m)\big].\label{eq:override}
\end{align}
Namely, $\rho(M)$ is the coordinate-wise application of $\rho$ to $M$ and $\sigma(M)$ is the column-wise application of $\sigma$ to $M$.

For $n\in\bbN$, we use $\mcS_n$ to denote the permutation group on $[n]$, i.e., $\mcS_n=\{\pi:[n]\to[n]:\pi([n])=[n]\}$. 
For any set $\mcX$, we define the group action of $\mcS_n$ on $\mcX^{m\times n}$ as follows: for $\pi\in\mcS_n$ and $X=[\bfx_1,\dots,\bfx_n]\in\mcX^{m\times n}$,
\begin{align}
\pi X\defeq [\bfx_{\pi(1)},\dots,\bfx_{\pi(n)}]. \label{eq:group_action}
\end{align}
For $m,n\in\bbN$, $\mcT\subset\mcS_n$, sets $\mcX,\mcY$, and $f:\mcX^{m\times n}\to\mcY^{m\times n}$, we say ``$f$ is $\mcT$-equivariant'' if $f(\pi X)=\pi f(X)$ for all $X\in\mcX^{m\times n}$ and $\pi\in\mcT$.
We say $f$ is ``permutation equivariant'' if $f$ is $\mcS_n$-equivariant.
When $\mcT=\{\pi\}$ is a singleton set, then we use the expression ``$\pi$-equivariant'' instead of ``$\mcT$-equivariant''.

\subsection{Floating-point arithmetic}\label{sec:float}

{\bf Floating-point numbers.} 
Throughout this paper, we consider IEEE 754 standard for floating-point arithmetic \cite{IEEE754}. 
For $\mbit, \ebit\in\bbN$,
we define $\fpq_{\mbit,\ebit}$ as the set of \emph{finite} floating-point numbers as follows:

\begin{align}
    \fpq_{\mbit, \ebit}\defeq \big\{&\!\pm\!(1.m_1\cdots m_{\mbit})\times2^{e},\pm(0.m_1\cdots m_{\mbit})\times2^{\emin}:\notag\\
    &m_1,\dots,m_{\mbit}\in\{0,1\},e\in[\emin,\emax]_\bbZ \big\} \label{eq:def:float}
\end{align}
where $1.m_1\cdots m_{\mbit}$ and $0.m_1\cdots m_{\mbit}$ are written in the binary representation, $\emin\defeq-2^{\ebit-1}+2$, and $\emax\defeq2^{\ebit-1}-1$.
Each floating-point number is represented by the \emph{sign} $s$, \emph{mantissa} $m_1,\dots,m_\mbit$, and \emph{exponent} $e$, i.e., $\mbit+\ebit+1$ bits are sufficient for representing all elements in $\fpq_{\mbit,\ebit}$. 

There are three additional (non-finite) floating-point numbers: \emph{positive infinity} $\infty$, \emph{negative infinity} $-\infty$, and \emph{not-a-number} $\nan$.
When we perform floating-point operations, $\infty$ and $-\infty$ may appear when the output exceeds the range of $\fpq_{\mbit,\ebit}$. $\nan$ may appear when the output cannot be approximated by an element in $\fpq_{\mbit,\ebit}\cup\{\infty,-\infty\}$, e.g., when we add $\infty$ to $-\infty$.
We assume that $-\infty,\infty$ follows the conventional order, i.e., $-\infty<x<\infty$ for all $x\in\bbR$. %
To denote the set of all floating-point numbers, we use $\overline{\fpq}_{\mbit,\ebit}\defeq\fpq_{\mbit,\ebit}\cup\{-\infty,\infty,\nan\}$.
We note that $\overline{\fpq}_{\mbit,\ebit}$ can also be represented by using $\mbit+\ebit+1$ bits; this is because
we are not using the whole $2^\ebit$ representations for the exponent in $\fpq_{\mbit,\ebit}$.
We use $\fmin \defeq 2^{\emin-\mbit}$ and $\fmax \defeq \lrp{2-2^{-\mbit}}\times 2^{\emax}$ to denote the smallest and largest positive floats, respectively.
For $x\in\fpq_{\mbit,\ebit}$, we use $x^-$ to denote the largest float that is strictly smaller than $x$, and we use $x^+$ to denote the smallest float that is strictly larger than $x$ (e.g., $\fmax^+=\infty$, $(-\fmax)^-=-\infty)$.
We often use $\fpq$ (and $\overline{\fpq}$) to represent $\fpq_{\mbit,\ebit}$ (and $\overline{\fpq}_{\mbit,\ebit}$) when $\mbit,\ebit$ are clear from the context.

Throughout this paper, we assume the following condition on $\mbit,\ebit$. We note that most practical floating-point formats (e.g., double/single/half-precision formats \cite{IEEE754}, bfloat16 \cite{bfloat}, FP8 (E5M2, E4M3) \cite{micikevicius2022fp8} satisfy this condition.
\begin{condition}
$2\le \mbit\le2^{\ebit-1}-3$.
\end{condition}

{\bf Rounding to floating-point numbers.}
To define floating-point operations, we define floating-point rounding $\round{\cdot}_{\fpq}:\bbR\cup\{-\infty,\infty,\nan\}\to\efpq$: %
\begin{equation*}
    \round{x}_{\fpq} \defeq \begin{cases}
        \argmin_{y\in \fpq} |x-y| &\text{ if } |x|<\fmax+2^{\emax-p-1}\!,
        \\ \infty &\text{ if } x \ge \fmax+2^{\emax-p-1}\!,
        \\ -\infty &\text{ if } x \le -\fmax-2^{\emax-p-1}\!,
        \\ \nan &\text{ if }x=\nan.
    \end{cases}
\end{equation*}
When there are two (finite) floating-point numbers that are closest and equidistant to $x\in\bbR$, we break the tie using the \emph{ties-to-even} rule \cite{IEEE754}: $\round{x}_{\fpq}$ is a unique float whose last mantissa bit $m_p$ is zero (see \cref{eq:def:float}).
If $\fpq$ is clear from the context, we use $\round{\cdot}$ to denote $\round{\cdot}_{\fpq}$. %

We define the \emph{correctly rounded version} of $\relu$ and $\exp$: for $\sigma\in\{\relu,\exp\}$ and $x\in\efpq$,
\begin{equation*}
    \round{\sigma}(x) \defeq \begin{cases}
        \round{\sigma(x)} &\text{ if } x\in \fpq,
        \\
        0  &\text{ if } x=-\infty, 
         \\  \infty  &\text{ if } x=\infty,
         \\ \nan &\text{ if }x=\nan.
    \end{cases} 
\end{equation*}
We note that the function values at $\pm\infty$ follows the convention induced by $\lim_{x\to-\infty}\sigma(x)=0$ and $\lim_{x\to\infty}\sigma(x)=\infty$ for $\sigma\in\{\relu,\exp\}$.

{\bf Floating-point operations.} We now define floating-point addition, subtraction, multiplication, and division:
for $x,y\in\fpq$, $x\oplus y\defeq\round{x+y}$, $x\ominus y\defeq\round{x-y}$, and $x\otimes y\defeq\round{x\times y}$. For $x,y\in\fpq$ with $y\ne0$, $x\oslash y\defeq\round{x/y}$.
We also define these operations to $x,y\in\efpq$ in \cref{sec:float-operation}, following the standard \cite{IEEE754}.
We note that, unlike conventional addition and multiplication, floating-point addition and multiplication are \emph{non-associative}: for $\odot\in\{\oplus,\otimes\}$,
\begin{align}
(x\odot y)\odot z\ne x\odot (y\odot z)%
\label{eq:non-associative}
\end{align}
in general.
Hence, we need to be careful about the ordering of operations.
As in the conventional operations, we first perform floating-point multiplications and divisions (first priority), and we next perform floating-point additions (second priority).
For operations with the same priority, we perform operations in the \emph{left associative} way, i.e., we perform operations from left to right.
For example, we have
$a\oplus b\otimes c\otimes d\oplus e=(a\oplus ((b\otimes c)\otimes d))\oplus e.$ 
When we add multiple floating-point numbers, we often use $\bigoplus$ as follows: for $n\in\bbN$ and $x_1,\dots,x_n\in\efpq$,
\begin{align}
\bigoplus_{i=1}^nx_i\defeq x_1\oplus x_2\oplus\cdots\oplus x_n.\label{eq:bigoplus}
\end{align}
Note that floating-point additions in \cref{eq:bigoplus} are computed in the left-associative way (i.e., from the left to right).
We also use $\oplus$ and $\bigoplus$ in the left-associative manner, e.g., %
\begin{align}
\bigoplus_{i=1}^nx_i\oplus\bigoplus_{j=1}^my_j=x_1\oplus\cdots\oplus x_n\oplus y_1\oplus\cdots\oplus y_m.
\label{eq:left-associative}
\end{align}
We lastly define floating-point matrix operations.
For $n,m\in\bbN$ and $M,N\in\efpq^{n\times m}$, $M\oplus N$ is the $n\times m$ floating-point matrix satisfying
$(M\oplus N)_{ij}=M_{ij}\oplus N_{ij}.$
Likewise, for $n,m,l\in\bbN$, $M\in\efpq^{n\times m}$, and $N\in\efpq^{m\times l}$, $M\otimes N$ is the $n\times l$ floating-point matrix satisfying $(M\otimes N)_{ij}=\bigoplus_{k=1}^m\,(M_{ik}\otimes N_{kj}).$

\subsection{Floating-point transformer}\label{sec:fp-transformer}

As in a real transformer \cite{vaswani2017attention}, we consider a floating-point transformer consisting of fully-connected networks and multi-head self-attention layers.
Specifically, given the hidden dimensions $d,r$, and the length of the input sequence $n$, the \emph{feed-forward networks} has the following form: for $\rho=\round{\text{ReLU}}$, %
$W_1\in\fpq^{r\times d},W_2\in\fpq^{d\times r}$, $\bfb_1\in\fpq^r,\bfb_2\in\fpq^d$, and $X\in\fpq^{d\times n}$,
\begin{align*}
\text{FF}_\theta^{r,d,n}(X)\:\!\!\defeq\:\!\!X\:\!\!\oplus\:\!\!(W_2\!\otimes\:\!\!\rho(W_1\!\otimes\:\!\! X\:\!\!\oplus\:\!\!(\bfb_1\onev_n^\top))\:\!\!\oplus\:\!\!(\bfb_2\onev_n^\top))
\end{align*}
where $\theta\in\fpq^{2dr+r+d}$ is vectorization of $W_1,W_2,b_1,b_2$ and
$\rho$ is applied in the coordinate-wise manner (see \cref{eq:override}).
Here, we note that the multiplication in $\bfb_1\onev_n^\top$ and $\bfb_2\onev_n^\top$ are exact.
To define the self-attention layer, we first define the (floating-point) softmax function $\sigma:\efpq^n\to\efpq^n$ as follows: for $\bfx=(x_1,\dots,x_n)\in\efpq^n$ and $x_{*}=\max_{i\in[n]}x_i$,\footnote{$x_*=\nan$ if $x_i=\nan$ for some $i$.}
\begin{align}
\sigma(\bfx)_i\!\defeq\!\round{\exp}\!(x_i\:\!\!-\:\!\!x_{*})\!\oslash\!\left(\bigoplus_{j=1}^n\round{\exp}\!(x_j\:\!\!-\:\!\!x_{*})\!\right)\!\!. \label{eq:floatsoftmaxdef}
\end{align}
Here, we subtract $x_{*}$ in the exponent following the real implementation of the softmax function in PyTorch \cite{paszke2019pytorch} and TensorFlow \cite{abadi2016tensorflow}; they subtract $x_{*}$ for better numerical stability.

Given the number of heads $h$ and $m\in\bbN$, the \emph{attention layer} has the following form: for $W_i^{K},W_{i}^Q,W_{i}^V\in\fpq^{m\times d}$ and $W_i^O\in\fpq^{d\times m}$ for all $i\in[h]$, for $X\in\fpq^{d\times n}$,
\begin{align}
\text{AT}_\phi^{h,m,d,n}\:\!\!(\:\!\!X\:\!\!)\!\defeq\!X\:\!\!\oplus\:\!\!\left(\:\!\!\bigoplus_{i=1}^hW_i^o\:\!\!\otimes\:\!\!\big(V_i\:\!\!\otimes\:\!\!\sigma(K_i^\top\:\!\!\!\!\otimes\:\!\! Q_i)\:\!\!\big)\!\!\right)\label{eq:attention}
\end{align}
where $\phi\in\fpq^{4hdm}$ is a vectorization of $W_i^K, W_i^Q, W_i^V, W_i^O$ for all $i\in[h]$, $K_i=W_i^K\otimes X$, $Q_i=W_i^Q\otimes X$,  $V_i=W_i^V\otimes X$, and $\sigma$ is applied column-wisely (see \cref{eq:override}).

Each \emph{transformer block} consists of a composition of an attention layer and a feed-forward network. Given $h,m,d,n\in\bbN$, we use $\mcB^{h,m,d,n}$ to denote all possible compositions of transformer blocks defined as follows:
\begin{align*}
\mcB^{h,m,r,d,n}\!\defeq&\{\text{FF}_{\theta_l}^{r,d,n}\!\!\!\circ\text{AT}_{\phi_l}^{h,m,d,n}\!\!\!\circ\cdots\circ \text{FF}_{\theta_1}^{r,d,n}\!\!\!\circ\text{AT}^{h,m,d,n}_{\phi_1}:\\
&\!l\in\bbN,\theta_i\in\fpq^{2dr+r+d},\phi_i\in\fpq^{4hdm},~\forall i\in[L]\}.
\end{align*}

Given the input dimension $d_\text{in}$, the output dimension $d_\text{out}$, a function $f:\fpq^{d_\text{in}\times n}\to\fpq^{d_\text{out}\times n}$ is a \emph{floating-point transformer} if there exists $W_\text{in}\in\fpq^{d\times d_\text{in}}$, $\bfb_\text{in}\in\fpq^{d}$, $W_\text{out}\in\fpq^{d_\text{out}\times d}$, $\bfb_\text{out}\in\fpq^{d_\text{out}}$, and $g\in\mcB^{h,m,d,n}$ such that 
\begin{align}
f(X)=W_\text{out}\:\!\!\otimes\:\!\! g(W_\text{in}\:\!\!\otimes\:\!\! X\oplus (\bfb_\text{in}\onev_n^\top))\:\!\!\oplus\:\!\! (\bfb_\text{out}\onev_n^\top).\label{eq:transformer}
\end{align}
We use ``real transformers'' to denote functions in \cref{eq:transformer} where all parameters can be real, and all floating-point operations are replaced by their exact counterparts.

\section{Main results}\label{sec:main_result}

In this section, we investigate whether the floating-point transformers inherit properties of the real transformers. Specifically, we study the following questions: for any input/output dimensions $\din,\dout$ and sequence length $n$, 
\begin{itemize}[leftmargin=0.19in]
\item Are floating-point transformers permutation equivariant? 
\item Can floating-point transformers represent all floating-point permutation-equivariant functions?
\end{itemize}
As we introduced in \cref{sec:contribution}, the answers to both questions are no. 
In the remainder of this section, we formally present the proofs to the answers and illustrate our intuition behind the proofs.

In \cref{sec:not-perm-eq}, we answer the first question by proving that floating-point transformers can represent a class of non-permutation-equivariant functions.
In \cref{sec:not-univ-rep}, we answer the second question by showing that there exists a permutation-equivariant function that cannot be represented by floating-point transformers when the sequence length $n$ is large. Furthermore, we show that floating-point transformers can represent all permutation-equivariant functions when $n$ is bounded. 

\subsection{Floating-point transformers are not $\mcS_n$-equivariant}\label{sec:not-perm-eq}
We first show that floating-point transformers can represent a class of non-permutation-equivariant functions by using the following theorem. To this end, for $n\ge2$, we consider a permutation $\pi_{(1,2)}^n\in\mcS_n$ such that 
$\pi_{(1,2)}^n(1)=2$, $\pi_{(1,2)}^n(2)=1$, and $\pi_{(1,2)}^n(i)=i$ for all $i\ge3$, i.e., $\pi_{(1,2)}^n$ swaps the first two coordinates.
\begin{theorem}\label{thm:diagonal-approx}
Let $\din,\dout,n\in\bbN$ such that $n\ge2$ and
$$\Delta_n\defeq\{[\bfx_1,\dots,\bfx_n]\in\fpq^{\din\times n}:\bfx_i\ne\bfx_j~\forall i\ne j\}.$$
Then, for any $\smash{\pi_{(1,2)}^n}$-equivariant $f^*:\Delta_n\to\fpq^{\dout\times n}$, there exists a floating-point transformer $f:\fpq^{\din\times n}\to\fpq^{\dout\times n}$ such that $f=f^*$ on $\Delta_n$.
\end{theorem}
\cref{thm:diagonal-approx} shows that floating-point transformers can represent all $\smash{\pi_{(1,2)}^n}$-equivariant functions $f^*$ from $\Delta_n$ to $\fpq^{\dout\times n}$. Note that $f^*$ may not be permutation equivariant since $\pi f(X)\ne f(\pi X)$ in general for all non-trivial permutations $\pi\ne\smash{\pi_{(1,2)}^n}$. %

To prove \cref{thm:diagonal-approx}, we exploit the non-associativity of the floating-point addition (see \cref{eq:non-associative}).
For example, consider two matrices $V\in\fpq^{m\times n}$, $\Sigma\in\fpq^{n\times n}$ and a non-trivial permutation $\pi\in\mcS_n$. %
Let $Z=V\Sigma$ and $\hat Z=V\otimes\Sigma$, 
i.e., $Z$ is computed by using the exact operations, but $\hat Z$ is computed by the floating-point matrix multiplication.
Let $\Sigma^\pi=\pi(\pi(\Sigma^\top))^\top$ be the row/column permuted version of $\Sigma$, i.e., $\Sigma^\pi_{ij}=\Sigma_{\pi^{-1}(i)\pi^{-1}(j)}$.
Then, one can observe that $\pi Z=(\pi V)\Sigma^\pi$. However, this is not the case for $\hat Z$ since
\begin{align*}
&(\pi\hat Z)_{ij}=\bigoplus_{k=1}^n(V_{ik}\otimes\Sigma_{k\pi^{-1}(j)})\vspace{-1in}\\[-0.1in]
&\ne \bigoplus_{k=1}^n(V_{i\pi^{-1}(k)}\otimes \Sigma_{\pi^{-1}(k)\pi^{-1}(j)})=((\pi V)\otimes \Sigma^\pi)_{ij}
\end{align*}
in general due to the non-associativity of the floating-point addition.
The proof of \cref{thm:diagonal-approx} is based on this observation, where we make non-permutation-equivariant floating-point attention layers by exploiting the $V_i\otimes\sigma(K_i^\top\otimes Q_i)$ part in \cref{eq:attention} (consider $V\leftarrow V_i$ and $\Sigma\leftarrow\sigma(K_i^\top\otimes Q_i)$).
Using those attention layers, we then construct a floating-point transformer that represents a target \smash{$\pi_{(1,2)}^n$}-equivariant function \smash{$f^*:\Delta_n\to\fpq^{\dout\times n}$}.
Here, we note that $\pi_{(1,2)}^n$-equivariance in our floating-point transformer construction is necessary; see \cref{sec:restriction} for more discussions.
We present the full proof of \cref{thm:diagonal-approx} in \cref{sec:pfthm:diagonal-approx}.

\subsection{Floating-point transformers cannot represent all $\mcS_n$-equivariant functions}\label{sec:not-univ-rep}
In \cref{sec:not-perm-eq}, we observed that floating-point transformers can represent a class of non-permutation-equivariant functions, i.e., they are not permutation equivariant in general. %
In this section, we first show that even some permutation-equivariant floating-point functions cannot be represented by floating-point transformers when the sequence length $n$ is large.
This observation is in contrast to the real transformers that can universally approximate permutation-equivariant $L^p$ functions on a compact domain, regardless of the value of $n$ \cite{Yun2020Are}.
We then show that floating-point transformers can represent all permutation-equivariant floating-point functions when the sequence length $n$ is properly bounded.

To describe our results, we first define the $(\alpha,\beta)$-similarity.
\begin{definition}
Let $d,n,\alpha,\beta\in\bbN$ such that $\alpha<\beta\le n$. We say $X=[\bfx_1,\dots,\bfx_n],Y=[\bfy_1,\dots,\bfy_n]\in\smash{\efpq^{d\times n}}$ are 
``$(\alpha,\beta)$-similar'' if there exist \smash{$\bfz_1,\bfz_2\in\efpq^d$} such that
\begin{align*}
\bfx_i=\begin{cases}
\bfz_1~&\text{if}~1\le i\le \alpha,\\
\bfz_2~&\text{if}~\alpha<i\le \beta,
\end{cases}\quad\bfy_i=\begin{cases}
\bfz_1~&\text{if}~1\le i< \alpha,\\
\bfz_2~&\text{if}~\alpha\le i\le \beta,
\end{cases}%
\end{align*}
and $\bfx_j=\bfy_j$ for all $j\in\{\beta+1,\dots,n\}$.
\end{definition}
One can observe that if $X$ and $Y$ are $(\alpha,\beta)$-similar, then they only differ at the $\alpha$-th column.

We now introduce the following theorem to show the existence of permutation-equivariant floating-point functions that cannot be represented by floating-point transformers. The proof of \cref{thm:counterexample} is presented in \cref{sec:pfthm:counterexample}. 
\begin{theorem}\label{thm:counterexample}
Let $\din,\dout,n,\alpha,\beta\in\bbN$ such that $\alpha\ge3\times2^\mbit$, $\beta-\alpha\ge6\times2^\mbit$, and $n\ge \beta$. 
Then, for any $(\alpha,\beta)$-similar $X,Y\in\fpq^{\din\times n}$ and for any floating-point transformer $f:\fpq^{\din\times n}\to\fpq^{\dout\times n}$, %
$f(X)$ and $f(Y)$ are also $(\alpha,\beta)$-similar. %
\end{theorem}
\cref{thm:counterexample} shows that if inputs $X,Y$ are $(\alpha,\beta)$-similar, then corresponding outputs $f(X),f(Y)$ are also $(\alpha,\beta)$-similar for all floating-point transformers $f$ when $\alpha\ge3\times2^\mbit$, $\beta-\alpha\ge6\times2^\mbit$, and $n\ge \beta$.
This implies that floating-point transformers cannot represent all permutation-equivariant floating-point functions when $n\ge9\times2^\mbit$ since they cannot represent a (non-empty) class of permutation-equivariant floating-point functions $f^*$ such that $f^*(X)$ and $f^*(Y)$ are not $(\alpha,\beta)$-similar for some $(\alpha,\beta)$-similar inputs $X,Y$.
We note that $9\times2^\mbit$ can be small, especially for low-precision formats: e.g., 8-bit formats E5M2 ($\mbit=2$), E4M3 ($\mbit=3$).

We prove \cref{thm:counterexample} using the following observation (\cref{lem:same-sum0,lem:same-sum2}): for any $a,b\in\efpq$, 
$\bigoplus_{i=1}^\alpha a\oplus\bigoplus_{i=\alpha+1}^\beta b=\bigoplus_{i=1}^{\alpha-1}a\oplus\bigoplus_{i=\alpha}^\beta b$
when $\alpha,\beta$ are large.\footnote{Note that the floating-point summation is computed in the left-associative way (see \cref{eq:left-associative}).}
Namely, for any $a,b\in\efpq$, even if we sequentially add different numbers of $a$ and $b$ under the floating-point arithmetic, the resulting output can be identical due to the round-off error.
Using this property, we show that if two inputs to the attention layers are $(\alpha,\beta)$-similar, then the corresponding outputs are also $(\alpha,\beta)$-similar.
Then, \cref{thm:counterexample} naturally follows.

We note that the result in \cref{thm:counterexample} only holds when $n$ (and $\alpha,\beta$) is large. On the other hand, when $n\le6\times2^\mbit-2$, floating-point transformers can represent all permutation-equivariant floating-point functions as stated in the following theorem. The proof of \cref{thm:perm-equiv-approx} is in \cref{sec:pfthm:perm-equiv-approx}.
\begin{theorem}\label{thm:perm-equiv-approx}
Let $\din,\dout,n\in\bbN$ such that $n\le 6\times2^\mbit-2$. 
Then, for any permutation-equivariant function $f^*:\fpq^{\din\times n}\to\fpq^{\dout\times n}$, there exists a floating-point transformer $f:\fpq^{\din\times n}\to\fpq^{\dout\times n}$ such that $f=f^*$ on $\fpq^{\din\times n}$.
\end{theorem}

\section{Discussions}\label{sec:discussion}
\subsection{Inductive biases in floating-point transformers}\label{sec:restriction}

In this section, we discuss inductive biases in floating-point transformers that are imposed by the definitions of the attention layers and fully-connected networks.
We first show that floating-point transformers must be $\pi_{(1,2)}^n$-equivariant.
The proof of \cref{thm:equiv} is presented in \cref{sec:pfthm:equiv}.
\begin{theorem}\label{thm:equiv}
Let $\din,\dout,n\in\bbN$ such that $n\ge2$. Then, for any floating-point transformer $f:\fpq^{\din\times n}\to\fpq^{\dout\times n}$, $f$ is $\pi_{(1,2)}^n$-equivariant.
\end{theorem}
\cref{thm:equiv} follows from the property of floating-point addition: for $x_1,\dots,x_n\in\fpq$, it holds that
\begin{align*}
\bigoplus_{i=1}^n x_i\!=\!(\cdots\!((x_1\oplus x_2)\oplus x_3)\oplus\cdots\oplus x_n)\!=\!\bigoplus_{i=1}^n x_{\pi_{(1,2)}^n\!(i)}.
\end{align*}
Based on this, one can easily observe that floating-point attention layers are $\pi_{(1,2)}^n$-equivariant, and hence, floating-point transformers are $\pi_{(1,2)}^n$-equivariant.

We also note that $\pi_{(1,2)}^n$-equivariance is the minimal equivariance structure in floating-point transformers. Namely, for any non-trivial permutation $\pi\in\mcS_n$ with $\pi\ne\pi_{(1,2)}^n$, there exists a floating-point transformer that is not $\pi$-equivariant.
Here, the existence of such a floating-point transformer is guaranteed by \cref{thm:diagonal-approx}.

We next show that for floating-point transformers, if some input coordinates are identical, then the corresponding output coordinates should also be the same. We note that such a property also exists in real transformers.
The proof of \cref{thm:eq-preserve} is in \cref{sec:pfthm:eq-preserve}
\begin{definition}
Let $\din,\dout,n\in\bbN$, $\mcX,\mcY$ be set, and $f:\mcX^{\din\times n}\to\mcY^{\dout\times n}$.
We say ``$f$ preserves equality'' if for any $X=[\bfx_1,\dots,\bfx_n]\in\mcX^{\din\times n}$ and for $f(X)=[\bfy_1,\dots,\bfy_n]$, $\bfy_i=\bfy_j$ whenever $\bfx_i=\bfx_j$.
\end{definition}
\begin{theorem}\label{thm:eq-preserve}
Floating-point transformers preserve equality.
\end{theorem}

\subsection{Effects of positional encoding}

In this section, we discuss the effects of positional encoding on the expressive power of floating-point transformers. 
Given the sequence length $n$, let $\bfp_1,\dots,\bfp_n\in\fpq^{\din}$ be the positional encoding vectors. 
A typical practice to encode the position in the input $X$ is to add the position vectors to corresponding position as $X+[\bfp_1,\dots,\bfp_n]$ \cite{vaswani2017attention}.
However, for any non-trivial choice of $\bfp_1,\dots,\bfp_n\in\fpq^{\din}$, some information in $X$ will be lost during the encoding due to the following lemma. We present the proof of \cref{lem:pos-enc} in \cref{sec:pflem:pos-enc}.
\begin{lemma}\label{lem:pos-enc}
For any $z\in\efpq\setminus\{0\}$, $x\mapsto x\oplus z$ is not injective on $\fpq$.
\end{lemma}
By \cref{lem:pos-enc}, one can observe that $X+[\bfp_1,\dots,\bfp_n]$ is not injective unless $\bfp_i=(0,\dots,0)$ for all $i\in[n]$. In other words, all floating-point transformers cannot distinguish $[\bfx_1,\dots,\bfx_n]$ and $[\bfy_1,\dots,\bfy_n]$ if $\bfx_i+\bfp_i=\bfy_i+\bfp_i$ for all $i\in[n]$.
This is in contrast to the real transformer case, where proper positional encoding strictly improves the expressive power of real transformers \cite{Yun2020Are}.

We can also encode the position by introducing an additional dimension to an input $X$, e.g., consider the positional encoding vector $\bfp=(1,\dots,n)\in\bbF^{n\times 1}$ and a new input $[X^\top,\bfp]^\top=[\bfx'_1,\dots,\bfx'_n]$.
In this case, $\bfx'_1,\dots,\bfx'_n$ are distinct, and hence, floating-point transformers can represent all functions from $\fpq^{\din\times n}\to\fpq^{\dout\times n}$ by \cref{thm:diagonal-approx}.

\section{Proofs}\label{sec:proof}
In this section, we present the proofs of \cref{thm:diagonal-approx,thm:counterexample,thm:perm-equiv-approx}.
For better readability, we prove \cref{thm:diagonal-approx,thm:perm-equiv-approx} first in \cref{sec:pfthm:diagonal-approx,sec:pfthm:perm-equiv-approx}, and prove \cref{thm:counterexample} later in \cref{sec:pfthm:counterexample}.

\subsection{Proof of \cref{thm:diagonal-approx}}\label{sec:pfthm:diagonal-approx}
In this proof, we explicitly construct a floating-point transformer $f$ that represents a $\pi_{(1,2)}^n$-equivariant target function $f^*:\Delta_n\to\fpq^{\dout\times n}$.
In particular, our floating-point transformer $f$ can be represented as
\begin{align}
f(X)=W_\text{out}\otimes (\psi\circ \phi(W_\text{in}\otimes X))\label{eq:diagonal-approx1}
\end{align}
for $X=[\bfx_1,\dots,\bfx_n]\in\fpq^{\din\times n}$ and
for some transformer blocks $\phi,\psi$. 
To this end, for some $d>\din,\dout$, we choose $W_\text{in}\in\fpq^{d\times\din}$ such that $(W_\text{in})_{ii}=1$ for all $i\in[\din]$ and all other entries are zero. %
Namely, $W_\text{in}\otimes X$ preserves the input information in the first $\din$ dimensions and pads zeros for the remaining $d-\din$ dimensions.
We will choose the explicit value of $d$ at the end of this proof.
Here, the $i$-th column of the output matrix of $\phi$ will hold the corresponding input $\bfx_i$ and the information that the input is either $X$ or $\pi_{(1,2)}^nX$.
As stated in the following lemma, this information in each column is sufficient to generate the $i$-th column of the target output since the target function $f^*$ is $\pi_{(1,2)}^n$ equivariant. The proof of \cref{lem:diagonal-factorize} is in \cref{sec:pflem:diagonal-factorize}.
\begin{lemma}\label{lem:diagonal-factorize}
Let $f^*:\Delta_n\to\fpq^{\dout\times n}$ be a $\smash{\pi_{(1,2)}^n}$ equivariant function. Then, there exists $\tilde f:\fpq^{\din}\times\fpq^{\din\times n}\to\fpq^{\dout}$ such that for each $X=[\bfx_1,\dots,\bfx_n]\in\fpq^{\din\times n}$, 
\begin{align*}
f^*(X)&=[\tilde f(\bfx_1,X),\dots,\tilde f(\bfx_n,X)]\\
&=[\tilde f(\bfx_1,\pi_{(1,2)}^nX),\dots,\tilde f(\bfx_n,\pi_{(1,2)}^nX)].
\end{align*}
\end{lemma}
In $\psi$, we use the information in each column of the output of $\phi$ to produce the corresponding column of the target output, in the first $\dout$ dimensions of the output of $\psi$. We lastly choose $W_\text{out}\in\fpq^{\dout\times d}$ as $(W_\text{out})_{ii}=1$ for all $i\in[\dout]$ and all other entries are zero so that $W_\text{out}\otimes Z=Z_{:\dout}$.

We now describe our constructions of $\phi,\psi$ and illustrate the main idea behind them.
To describe the precise function of $\phi$, we introduce the following lemma. The proof of \cref{lem:three-max} is in \cref{sec:pflem:three-max}.
\begin{lemma}\label{lem:three-max}
Let $n\ge2$ and $f:\mcS_n\to[n]^{\binom{n}{3}}$ defined as
$$f(\pi)=\Big(\argmax_{i\in\{i_1,i_2,i_3\}}\pi(i)\Big)_{1\le i_1<i_2<i_3\le n}.$$
Then, for each $\bfx\in f(\mcS_n)$, $f^{-1}(\bfx)=\{\pi,\pi_{(1,2)}^n\circ\pi\}$ for some $\pi\in\mcS_n$. %
\end{lemma}
\cref{lem:three-max} states that for any permutation $\pi\in\mcS_n$, if we know $\max\{\pi(i_1),\pi(i_2),\pi(i_3)\}$ for all $1\le i_1<i_2<i_3\le n$, then we can recover the input $\pi$ up to its $\pi_{(1,2)}^n$-equivariant counterpart. This implies that for an input $X=[\bfx_1,\dots,\bfx_n]\in\Delta_n$ to a floating-point transformer, if we identify the set of input tokens $\{\bfx_1,\dots,\bfx_n\}$ and the element with the largest index (i.e., the position in the sequence) for all distinct $\bfx,\bfx',\bfx''\in\{\bfx_1,\dots,\bfx_n\}$, then we can assert that the input is either $X$ or $\pi_{(1,2)}^nX$, which is sufficient to generate the output token by \cref{lem:diagonal-factorize}.

Based on this observation, in the following lemma, we design an attention layer that collects the input information and saves it in each column. 
In \cref{lem:diagonal-attention}, we use $\pi_1,\pi_2,\pi_3\in\mcS_3$ to denote $\pi_i(x)=x+(i-1)~\text{mod}~3$ and $\{\bfz_{1,1},\bfz_{1,2},\bfz_{1,3}\},\dots,\{\bfz_{\gamma,1},\bfz_{\gamma,2},\bfz_{\gamma,3}\}$ to denote all triplets of distinct elements in $\fpq^{\din}$ where $\gamma=\binom{|\fpq|^{\din}}3$.
We present the proof of \cref{lem:diagonal-attention} in \cref{sec:pflem:diagonal-attention}.
\begin{lemma}\label{lem:diagonal-attention}
Let $h=1$, $\din,n\in\bbN$, and $e=\din+3\gamma$. There exists $t_d,t_m,t_r\in\bbN$ such that for any $d\ge t_d$, $m\ge t_m$, and $r\ge t_r$,
there exists $\phi\in\mcB^{h,m,r,d,n}$ 
satisfying the following properties.
For each $Z\in\fpq^{d\times n}$ with $Z_{:\din}=[\bfx_1,\dots,\bfx_n]\in\Delta_{\din}$ and $Z_{\din+1:}=\zerov_{d-\din,n}$,  $j\in[\gamma-1]\cup\{0\}$, and $k\in[3]$,\vspace{-0.1in}
\begin{itemize}[leftmargin=0.12in]
    \item $\phi(Z)_{:\din}=[\bfx_1,\dots,\bfx_n]$ and $\phi(Z)_{e+1:}=\zerov_{d-e,n}$,
    \item $\phi(Z)_{\din+3j+k}=\delta\times\onev_n^\top$ if there exists $i_1\!<i_2\!<i_3$ such that $\{\bfx_{i_1},\bfx_{i_2}\}=\{\bfz_{j,{\pi_k(1)}},\bfz_{j,{\pi_k(2)}}\}$ and $\bfx_{i_3}=\bfz_{j,{\pi_k(3)}}$,
    \item $\phi(Z)_{\din+3j+k}\in(\fpq\setminus\{\delta\})^{1\times n}$ otherwise, %
    \vspace{-0.1in}
\end{itemize}
where $\delta=3^{++}$ %
if $\mbit\ge3$ and $\delta=3^{+}$ %
if $\mbit=2$.
\end{lemma}

For $d\ge t_d$, $m\ge t_m$, and $r\ge t_r$, let $\phi_{d,m,r}\in\mcB^{1,m,r,d,n}$ be the transformer block in \cref{lem:diagonal-attention}. Then, the $i$-th column of the output of $\phi_{d,m,r}$ contains $\bfx_i$ in the first $\din$ coordinates. For the next $e-\din$ coordinates, it contains the information about $\{\bfx_1,\dots,\bfx_n\}$ (consider the union of all $\{\bfz_{j,1},\bfz_{j,2},\bfz_{j,3}\}$ where $\phi_{d,m,r}(W_\text{in}\otimes X)_{d_\text{in}+j+k}=\delta\times\onev_n^\top$ for some $k\in[3]$).
Furthermore, it also contains the information about the element with the largest index for all distinct $\bfx,\bfx',\bfx''\in\{\bfx_1,\dots,\bfx_n\}$. %
To prove \cref{lem:diagonal-attention}, we use the following non-associativity of floating-point addition: e.g., for $\mbit\ge3$,
$(1^+\oplus1^{++})\oplus1^{++}=(1^{++}\oplus1^{+})\oplus1^{++}=3^{+++}$ but $(1^{++}\oplus1^{++})\oplus1^{+}=3^{++}.$ 
Specifically, we identify the last element of three distinct input tokens in an attention layer by using this property. See \cref{sec:pflem:diagonal-attention} for the full construction.

By \cref{lem:diagonal-factorize,lem:three-max}, one can observe that there exists a column-wise function $\tilde \psi:\fpq^{e}\to\fpq^{\dout}$ such that 
$\tilde \psi\circ \phi_{d,m,r}(W_\text{in}\otimes X)_{:e}=f^*(X)$.
Here, we note that $\phi_{d,m,r}(W_\text{in}\otimes X)_{:e}$ is independent of $d,m,r$.
We now generate the target output using the following theorem. We present the proof of \cref{lem:token-univ-approx} in \cref{sec:pflem:token-univ-approx}, which is based on the construction in \cite{hwang2025floating}.
\begin{lemma}\label{lem:token-univ-approx}
Let $d_1,d_2,m,h,n\in\bbN$. There exist $t_d',t_r'\in\bbN$ such that for any $d\ge t_d'$, $r\ge t_r'$, and $\psi^*:\fpq^{d_1}\to\fpq^{d_2}$, there exists 
$\psi\in\mcB^{h,m,r,d,n}$ satisfying the following property. 
For each $Z\in\fpq^{d\times n}$ with $Z_{:d_1}=[\bfx_1,\dots,\bfx_n]$ and $Z_{d_1+1:}=\zerov_{d-d_1,n}$, 
$\psi(Z)_{:d_2}=[\psi^*(\bfx_1),\dots,\psi^*(\bfx_n)]$ and $\psi(Z)_{d_2+1:}=\zerov_{d-d_2,n}$.
\end{lemma}
By \cref{lem:token-univ-approx}, for any large enough $d,r$ and any $m\in\bbN$, there exists a transformer block $\psi_{d,m,r}\in\mcB^{1,m,r,d,n}$ such that $\psi_{d,m,r}\circ \phi_{d,m,r}(W_\text{in}\otimes X)_{:\dout}=f^*(X)$.
Choose $d\ge\max\{t_d,t_d'\}$, $m\ge t_m$, and $r\ge\max\{t_r,t_r'\}$ where $t_d,t_m,t_r$ are from \cref{lem:diagonal-attention} and $t_d',t_r'$ are from \cref{lem:token-univ-approx}, and choose $\phi=\phi_{d,m,r}$ and $\psi=\psi_{d,m,r}$.
As we mentioned at the beginning of this proof, we construct $f$ as in \cref{eq:diagonal-approx1}. This completes the proof.

We lastly note that our floating-point transformer constructions in this proof use $h=1$ and $m,r,d=O(\mathtt{poly}(|\fpq|,n))$; here, we treat $\din,\dout$ as constants. 
However, $m,r,d$ can be reduced to $m=r=1$ and $d=\din+\dout+c$ for some constant $c\in\bbN$ by applying techniques that transform shallow and wide networks into deep and narrow ones. See \cref{sec:deep-narrow} for more details.

\subsection{Proof of \cref{thm:perm-equiv-approx}}\label{sec:pfthm:perm-equiv-approx}
In this proof, we construct a floating-point transformer $f$ that represents a permutation equivariant function $f^*:\fpq^{\din\times n}\to\fpq^{\dout\times n}$ when $n\le6\times2^\mbit-2$.
Our construction can be represented as 
\begin{align}
f(X)=W_\text{out}\otimes (\psi\circ \phi(W_\text{in}\otimes X))\label{eq:pfthm:permutation-approx1}
\end{align}
for some transformer blocks $\phi,\psi$. 
As in the proof of \cref{thm:diagonal-approx}, we also choose $W_{\text{in}}\in\fpq^{d\times\din}$ and $W_{\text{out}}\in\fpq^{\dout\times d}$ so that $W_{\text{in}}\otimes(x_1,\dots,x_{\din})=(x_1,\dots,x_{\din},0,\dots,0)$ and $W_{\text{out}}\otimes(x_1,\dots,x_{d})=(x_1,\dots,x_{\dout})$ where $d$ will be chosen at the end of this proof.
To describe the functions of $\phi,\psi$, we introduce the following lemma. 
The proof of \cref{lem:permutation-factorize} is in \cref{sec:pflem:permutation-factorize}.
\begin{lemma}\label{lem:permutation-factorize}
Let
$f^*:\fpq^{\din\times n}\to\fpq^{\dout\times n}$ be a permutation equivariant function. Then, there exists $\tilde f:\fpq^{\din}\times\bbN^{|\fpq|^{\din}}\to\fpq^{\dout}$ such that for each $X=[\bfx_1,\dots,\bfx_n]\in\fpq^{\din\times n}$, 
\begin{align*}
f^*(X)=[\tilde f(\bfx_1,\bfc_X),\dots,\tilde f(\bfx_n,\bfc_X)]
\end{align*}
where \smash{$\bfc_X\in(\bbN\cup\{0\})^{|\fpq|^{\din}}$} is a vector whose $i$-th coordinate is the number of columns of $X$ that are equal to $\bfz_i$.
\end{lemma}
\cref{lem:permutation-factorize} states that for any permutation equivariant function $f^*$, if we know the $i$-th input token and the distribution of inputs, then we can generate the $i$-th column of the output.
Using the following lemmas, we design a transformer block $\phi$ such that the $i$-th column of the output of $\phi$ contains the distribution of inputs and the $i$-th input token. This information will be used in $\psi$ to generate the target output. 
In \cref{lem:permutation-attention}, $\bfz_1,\dots,\bfz_{|\fpq|^{\din}}$ denotes all vectors in $\fpq^{\din}$.
The proofs of \cref{lem:permutation-attention,lem:max-distinguish} are in \cref{sec:pflem:permutation-attention,sec:pflem:max-distinguish}, respectively.
\begin{lemma}\label{lem:permutation-attention}
Let $h=1$ and $e=\din+\fpq^{|\din|}$.
Then, there exist $t_d,t_m,t_r\in\bbN$ such that for any $d\ge t_d,m\ge t_m$, and $r\ge t_r$,
there exists $\phi\in\mcB^{h,m,r,d,n}$  satisfying the following property.
For each $Z\in\fpq^{d\times n}$ satisfying $Z_{:\din}=[\bfx_1,\dots,\bfx_n]\in\fpq_{\din}$ and $Z_{\din+1:}=\zerov_{d-\din,n}$, \vspace{-0.1in}
\begin{itemize}[leftmargin=0.12in]
    \item $\phi(Z)_{:\din}=[\bfx_1,\dots,\bfx_n]$ and $\phi(Z)_{e+1:}=\zerov_{d-e,n}$,
    \item for $j\in[e]\setminus[\din]$, $\phi(Z)_{j}=(\bigoplus_{i=1}^{k_j}1^+)\times\onev_n^\top$ where $k_j$ denotes the number of $i\in[n]$ such that $\bfx_i=\bfz_{j-\din}$.
\end{itemize}
\end{lemma}
\begin{lemma}\label{lem:max-distinguish}
$\bigoplus_{i=1}^{k}1^+$ for all $k\in[3\times2^\mbit-1]\cup\{0\}$ are distinct and $\bigoplus_{i=1}^{3\times2^\mbit-1+n}1^+=2^{\mbit+2}$ for all $n\in\bbN\cup\{0\}$.
\end{lemma}
\cref{lem:permutation-attention,lem:max-distinguish} show that for each $d\ge t_d,m\ge t_m$, and $r\ge t_r$, there exists $\phi_{d,m,r}\in\mcB^{1,m,r,d,n}$ for such that each column of the output of $\phi_{d,m,r}$ contains the distribution of inputs and the corresponding input token.
Here, the input distribution can be recovered from $\bigoplus_{i=1}^{k_j}1^+$ for all $j$ by using \cref{lem:max-distinguish}. If $\bigoplus_{i=1}^{k_j}1^+<2^{p+2}$ for all $j$, then each output column of $\phi_{d,m,r}$ contains the exact distribution of inputs.
If $\bigoplus_{i=1}^{k_j}1^+=2^{p+2}$ for two $j$s (say $j_1,j_2$), then we can conclude that $X=[\bfx_1,\dots,\bfx_n]$ contains exactly $3\times2^{p+1}-1$ columns of $\bfz_{j_1}$ and $3\times2^{p+1}-1$ columns of $\bfz_{j_2}$.
If $\bigoplus_{i=1}^{k_j}1^+<2^{p+2}=2^{p+2}$ for a single $j$ (say $j_3$), then we can recover the number of columns in $X$ that is equal to $\bfz_{j_3}$ since we have the exact count of columns in $X$ that are equal to $\bfz_j$ for all $j\ne j_3$.
This implies that there exists $\tilde\psi:\fpq^e\to\fpq^{\dout}$ such that
$\tilde\psi\circ\phi_{d,m,r}(W_\text{in}X)_{:e}=f^*(X)$. As in the proof of \cref{thm:diagonal-approx}, $\phi_{d,m,r}(W_\text{in}X)_{:e}$ is independent of $d,m,r$.

We then apply \cref{lem:token-univ-approx} with $\psi^*\leftarrow\tilde\psi$.
This implies that for any $d\ge\max\{t_d,t_d'\},r\ge\max\{t_r,t_r'\}$, $m\ge t_m$ ($t_d',t_r'$ are from \cref{lem:token-univ-approx} and $t_d,t_m,t_r$ are from \cref{lem:permutation-attention}), there exists $\psi_{d,m,r}$ such that $W_\text{out}\otimes(\psi_{d,m,r}\circ\phi_{d,m,r}(W_\text{in}X))=f^*(X)$.
Then, choosing $\phi=\phi_{d,m,r}$ and $\psi=\psi_{d,m,r}$ for $d\ge\max\{t_d,t_d'\},r\ge t_r'$, $m\ge t_m$ completes the proof. 
As in the proof of \cref{thm:diagonal-approx}, our floating-point transformer construction here can be modified to have $m,r=1$ and $d=\din+\dout+c$ for some constant $c\in\bbN$. See \cref{sec:deep-narrow} for more details.

\subsection{Proof of \cref{thm:counterexample}}\label{sec:pfthm:counterexample}

Let $\alpha,\beta\in\bbN$ such that $\alpha\ge3\times2^\mbit$, $\beta-\alpha=6\times2^\mbit$, and $n\ge\beta$.
Since fully-connected networks in a floating-point transformer are applied column-wise, they preserve the $(\alpha,\beta)$-similarity of inputs: if their inputs are $(\alpha,\beta)$-similar, then the corresponding outputs are also $(\alpha,\beta)$-similar.
Likewise, for any $g(Z)$ that preserves the $(\alpha,\beta)$-similarity of inputs and for any floating-point matrix $W$, functions like $Z\mapsto W\otimes Z$ and $Z\mapsto Z\oplus g(Z)$ also preserve the $(\alpha,\beta)$-similarity.
Hence, to prove \cref{thm:counterexample}, it suffices to show that for any $W^V,W^K,W^Q\in\fpq^{m\times d}$, the function $g:Z\mapsto(W^V\otimes Z)\otimes\sigma((W^K\otimes Z)^\top\otimes(W^Q\otimes Z))$ preserves the $(\alpha,\beta)$-similarity of inputs.

For $(\alpha,\beta)$-similar $Z,Z'\in\efpq^{d\times n}$,
let $S=[\bfs_1,\dots,\bfs_n]=\sigma((W^K\otimes Z)^\top\otimes(W^Q\otimes Z))$ and $V=[\bfv_1,\dots,\bfv_n]=W^V\otimes Z$, i.e., $g(Z)=V\otimes S$.
We also define $S',\bfs_j',V',\bfv_j'$ analogously.
By the definition of the floating-point softmax $\sigma$ (\cref{eq:floatsoftmaxdef}), one can observe that $\bfs_j,\bfs_j'$ are also $(\alpha,\beta)$-similar for all $j\in[n]$.
Furthermore, we have
\begin{align*}
g(Z)_i=\left[\bigoplus_{j=1}^n(V_{ij}\otimes s_{1,j}),\dots,\bigoplus_{j=1}^n(V_{ij}\otimes s_{n,j})\right]
\end{align*}
where $(s_{k,1},\dots,s_{k,n})=\bfs_k$. %
We also have a similar equation for $g(Z')_i$.
Since $V,V'$ and $\bfs_j,\bfs_j'$ are $(\alpha,\beta)$-similar, we can observe that each coordinate of $g(Z)_i$ and $g(Z')_i$ can be written as
\begin{align}
g(Z)_{ij}&=\bigoplus_{k=1}^{\alpha} a_{i,j}\oplus\!\bigoplus_{k=\alpha+1}^\beta b_{i,j}\oplus\!\bigoplus_{k=\beta+1}^{n} (V_{ik}\otimes s_{j,k}),\notag\\
g(Z')_{ij}&=\bigoplus_{k=1}^{\alpha-1} a_{i,j}\oplus\bigoplus_{k=\alpha}^\beta b_{i,j}\oplus\!\bigoplus_{k=\beta+1}^{n} (V_{ik}\otimes s_{j,k}),\label{eq:pfthm:counterexample1}
\end{align}
where $a_{i,j}=V_{ik}\otimes s_{k,j}=V_{ik'}'\otimes s_{k',j}'$ for all $k\in[\alpha],k'\in[\alpha-1]$ and $b_{i,j}=V_{ik}\otimes s_{k,j}=V_{ik'}'\otimes s_{k',j}'$ for all $k\in[\beta]\setminus[\alpha],k'\in[\beta]\setminus[\alpha-1]$.
Note that the floating-point summations in \cref{eq:pfthm:counterexample1} are computed in the left-associative way (see \cref{eq:left-associative}).

If all operations are exact (i.e., not floating-point based), then $g(Z)_i$ and $g(Z')_i$ are not $(\alpha,\beta)$-similar in general since the numbers of $a_{i,j}$ and $b_{i,j}$ are different in \cref{eq:pfthm:counterexample1} for $g(Z)_{ij}$ and $g(Z')_{ij}$.
However, this is not the case for the floating-point addition since repeated addition of the same number will saturate due to the round-off error, i.e., for large enough $\alpha,\beta$, it holds that
$\bigoplus_{k=1}^{\alpha} a_{i,j}\oplus\bigoplus_{k=\alpha+1}^\beta b_{i,j}=\bigoplus_{k=1}^{\alpha-1} a_{i,j}\oplus\bigoplus_{k=\alpha}^\beta b_{i,j}$.
We formally prove this intuition in the following lemmas. The proofs of \cref{lem:same-sum0,lem:same-sum2} are presented in \cref{sec:pflem:same-sum0,sec:pflem:same-sum2}.

\begin{lemma}\label{lem:same-sum0}
For any $x\in\efpq$ and $n,m\in\bbN\cup\{0\}$, it holds that 
$\bigoplus_{i=1}^{3\times2^{p}-1+n}x=\bigoplus_{i=1}^{3\times2^{p}-1+m}x$.
Furthermore, $\bigoplus_{i=1}^{3\times2^{p}-1}x\in \{0,\pm\infty,\nan\}\cup(\{\pm2^e:e\in\bbZ\}\cap\fpq)$.
\end{lemma}
\begin{lemma}\label{lem:same-sum2}
For any $x\in\efpq$, $z\in\{0,\pm\infty,\nan\}\cup(\{\pm2^e:e\in\bbZ\}\cap\fpq)$, 
and $n,m\in\bbN\cup\{0\}$,
it holds that
$z\oplus\bigoplus_{i=1}^{6\times2^p+n}x=z\oplus\bigoplus_{i=1}^{6\times2^p+m}x$.
\end{lemma}

By \cref{lem:same-sum0,lem:same-sum2} and our choices of $\alpha,\beta$, one can observe that $g(Z)$ and $g(Z')$ are $(\alpha,\beta)$-similar. This completes the proof.

\section{Conclusion}\label{sec:conclusion}
In this work, we present the first theoretical results on the expressive power of floating-point transformers. Unlike real transformers, we prove that floating-point transformers are neither permutation equivariant nor expressive enough to represent all permutation-equivariant sequence-to-sequence functions when the input sequence length is large.
On the other hand, we show that the floating-point transformers can represent all permutation-equivariant sequence-to-sequence functions only when the sequence length is bounded.
We also discuss the inductive biases and effects of positional encoding, and compare them with the real transformer cases.
We believe that our results help in understanding transformers implemented on computers.

\bibliography{reference}

@STRING{NeurIPS = "Annual Conference on Neural Information Processing Systems (NeurIPS)"}

@STRING{ICLR = "International Conference on Learning Representations (ICLR)"}

@STRING{ICML = "International Conference on Machine Learning (ICML)"}

@STRING{AISTATS = "International Conference on Artificial Intelligence and Statistics (AISTATS)"}

@STRING{AAAI = "AAAI Conference on Artificial Intelligence (AAAI)"}

@inproceedings{vaswani2017attention,
  title={Attention is all you need},
  author={Vaswani, Ashish and Shazeer, Noam and Parmar, Niki and Uszkoreit, Jakob and Jones, Llion and Gomez, Aidan N and Kaiser, {\L}ukasz and Polosukhin, Illia},
  booktitle=NeurIPS,
  year={2017}
}

@inproceedings{
dosovitskiy2021an,
title={{An Image is Worth 16x16 Words: Transformers for Image Recognition at Scale}},
author={Alexey Dosovitskiy and Lucas Beyer and Alexander Kolesnikov and Dirk Weissenborn and Xiaohua Zhai and Thomas Unterthiner and Mostafa Dehghani and Matthias Minderer and Georg Heigold and Sylvain Gelly and Jakob Uszkoreit and Neil Houlsby},
booktitle=ICLR,
year={2021}
}

@inproceedings{ying2021transformers,
  title={{Do transformers really perform badly for graph representation?}},
  author={Ying, Chengxuan and Cai, Tianle and Luo, Shengjie and Zheng, Shuxin and Ke, Guolin and He, Di and Shen, Yanming and Liu, Tie-Yan},
  journal=NeurIPS,
  year={2021}
}

@inproceedings{zhou2021informer,
  title={{Informer: Beyond efficient transformer for long sequence time-series forecasting}},
  author={Zhou, Haoyi and Zhang, Shanghang and Peng, Jieqi and Zhang, Shuai and Li, Jianxin and Xiong, Hui and Zhang, Wancai},
  booktitle=AAAI,
  year={2021}
}

@inproceedings{lee2019set,
  title={{Set transformer: A framework for attention-based permutation-invariant neural networks}},
  author={Lee, Juho and Lee, Yoonho and Kim, Jungtaek and Kosiorek, Adam and Choi, Seungjin and Teh, Yee Whye},
  booktitle=ICML,
  year={2019}
}

@inproceedings{
Yun2020Are,
title={{Are Transformers universal approximators of sequence-to-sequence functions?}},
author={Chulhee Yun and Srinadh Bhojanapalli and Ankit Singh Rawat and Sashank Reddi and Sanjiv Kumar},
booktitle=ICLR,
year={2020}
}

@inproceedings{alberti2023sumformer,
  title={{Sumformer: Universal approximation for efficient transformers}},
  author={Alberti, Silas and Dern, Niclas and Thesing, Laura and Kutyniok, Gitta},
  booktitle={Topological, Algebraic and Geometric Learning Workshops},
  year={2023}
}

@inproceedings{
kim2023provable,
title={{Provable Memorization Capacity of Transformers}},
author={Junghwan Kim and Michelle Kim and Barzan Mozafari},
booktitle=ICLR,
year={2023}
}

@article{park2024expressive,
  title={{Expressive power of ReLU and step networks under floating-point operations}},
  author={Park, Yeachan and Hwang, Geonho and Lee, Wonyeol and Park, Sejun},
  journal={Neural Networks},
  volume={175},
  pages={106297},
  year={2024},
  publisher={Elsevier}
}

@inproceedings{hwang2025floating,
  title={Floating-point neural networks can represent almost all floating-point functions},
  author={Hwang, Geonho and Park, Yeachan and Lee, Wonyeol and Park, Sejun},
  booktitle=ICML,
  year={2025}
}

@inproceedings{hwang2025interval,
  title={{Floating-Point Neural Networks Are Provably Robust Universal Approximators}},
  author={Hwang, Geonho and Lee, Wonyeol and Park, Yeachan and Park, Sejun and Saad, Feras},
  booktitle={International Conference on Computer Aided Verification},
  year={2025}
}

@ARTICLE{IEEE754,
  author={IEEE},
  journal={IEEE Std 754-2019 (Revision of IEEE 754-2008)}, 
  title={{IEEE Standard for Floating-Point Arithmetic}}, 
  year={2019},
  volume={},
  number={},
  pages={1-84},
  doi={10.1109/IEEESTD.2019.8766229}}

@article{micikevicius2022fp8,
  title={Fp8 formats for deep learning},
  author={Micikevicius, Paulius and Stosic, Dusan and Burgess, Neil and Cornea, Marius and Dubey, Pradeep and Grisenthwaite, Richard and Ha, Sangwon and Heinecke, Alexander and Judd, Patrick and Kamalu, John and others},
  journal={arXiv preprint arXiv:2209.05433},
  year={2022}
}

@article{abadi2016tensorflow,
  title={{Tensorflow: Large-scale machine learning on heterogeneous distributed systems}},
  author={Abadi, Mart{\'\i}n and Agarwal, Ashish and Barham, Paul and Brevdo, Eugene and Chen, Zhifeng and Citro, Craig and Corrado, Greg S and Davis, Andy and Dean, Jeffrey and Devin, Matthieu and others},
  journal={arXiv preprint arXiv:1603.04467},
  year={2016}
}

@article{paszke2019pytorch,
  title={{Pytorch: An imperative style, high-performance deep learning library}},
  author={Paszke, Adam and Gross, Sam and Massa, Francisco and Lerer, Adam and Bradbury, James and Chanan, Gregory and Killeen, Trevor and Lin, Zeming and Gimelshein, Natalia and Antiga, Luca and others},
  journal=NeurIPS,
  year={2019}
}

@misc{bfloat,
  title        = {Improve your model's performance with bfloat16},
  author       = {Google},
  note         = {\url{https://cloud.google.com/tpu/docs/
bfloat16}}
}

@article{yun2020n,
  title={O (n) connections are expressive enough: Universal approximability of sparse transformers},
  author={Yun, Chulhee and Chang, Yin-Wen and Bhojanapalli, Srinadh and Rawat, Ankit Singh and Reddi, Sashank and Kumar, Sanjiv},
  journal=NeurIPS,
  year={2020}
}

@inproceedings{takakura2023approximation,
  title={{Approximation and estimation ability of transformers for sequence-to-sequence functions with infinite dimensional input}},
  author={Takakura, Shokichi and Suzuki, Taiji},
  booktitle=ICML,
  year={2023}
}

@inproceedings{
mahdavi2024memorization,
title={{Memorization Capacity of Multi-Head Attention in Transformers}},
author={Sadegh Mahdavi and Renjie Liao and Christos Thrampoulidis},
booktitle=ICLR,
year={2024}
}

@InProceedings{pmlr-v258-dana25a,
  title = 	 {{Memorization in Attention-only Transformers}},
  author =       {Dana, L{\'e}o and Pydi, Muni Sreenivas and Chevaleyre, Yann},
  booktitle =AISTATS,
  year = 	 {2025}
}

@inproceedings{
furuya2025transformers,
title={{Transformers are Universal In-context Learners}},
author={Takashi Furuya and Maarten V. de Hoop and Gabriel Peyr{\'e}},
booktitle=ICLR,
year={2025}
}
\bibliographystyle{plainnat}

\newpage
\onecolumn
\appendix

\section{Additional notations and definitions}
For a logical proposition $P$, we define $\indcc{P}=1$ if $P$ is true and $\indcc{P}=0$ if $P$ is false.
We use $\rho$ to denote $\round{\relu}$.
We use $\fpq_-$ and $\fpq_+$ to denote the set of negative finite floats and the set of positive finite floats, respectively.

\begin{definition}\label{def:affine}
Let $d,d'\in\bbN$ and $f:\efpq^d\to\efpq^{d'}$. We say $f$ is a ``floating-point affine transformation'' if there exists $W\in\fpq^{d'\times d}$ and $b\in\fpq$ such that $f(\bfx)=W\otimes\bfx+b$ for all $\bfx\in\efpq^d$.
\end{definition}

\begin{definition}\label{def:distinguishability}
Let $d\in\bbN$, $\mcX\subset\fpq^d$, and $\mcY\subset\efpq$.
We say ``$\mcX$ is distinguishable with range $\mcY$'' if for any $\bfx,\bfx'\in\mcX$, there exists a floating-point affine transformation $\phi:\fpq^d\to\efpq$ such that 
$$\rho(\phi(\bfx))\ne\rho(\phi(\bfx'))~~\text{and}~~\rho(\phi(\mcX))\subset\mcY.$$
\end{definition}

\subsection{Floating-point operations}\label{sec:float-operation}
In this section, we define floating-point addition, subtraction, multiplication, and division on $\efpq\times\efpq$. Here, we only define $x\oslash y$ for $0\le x\le y<\infty$ with $y>0$, or $x=\nan$ since the other cases do not appear in this paper.
This is because the division only arises during the computation of the floating-point softmax function. Recall the definition of the floating-point softmax function $\sigma:\efpq^n\to\efpq^n$ in \cref{eq:floatsoftmaxdef}: for $\bfx=(x_1,\dots,x_n)\in\efpq^n$ and $x_{*}=\max_{i\in[n]}x_i$,
\begin{align*}
\sigma(\bfx)_i\!\defeq\!\round{\exp}\!(x_i\:\!\!-\:\!\!x_{*})\!\oslash\!\left(\bigoplus_{j=1}^n\round{\exp}\!(x_j\:\!\!-\:\!\!x_{*})\!\right)\!\!. 
\end{align*}
Here, if $x_i\in\{\nan,\infty\}$ for some $i\in[n]$, then $x_i-x_*=\nan$, i.e., division by $\nan$ occurs.
Likewise, if $x_i=-\infty$ for all $i\in[n]$, then the division by $\nan$ occurs.
If $x_i\in\fpq\cup\{-\infty\}$ for all $i\in[n]$ and there exists $x_j\ne-\infty$ for some $j\in[n]$,
then the denominator in the computation of the floating-point softmax is non-zero and finite, while the numerator is smaller than the denominator.
Under this observation, we now define floating-point addition, subtraction, multiplication, and division on $\efpq\times\efpq$ as follows:
\begin{align*}
x\oplus y=y\oplus x&\defeq\begin{cases}
\round{x+y} &\text{ if } (x,y)\in\fpq^2,\\
\infty &\text{ if } (x,y)\in(\fpq\cup\{\infty\})\times\{\infty\},\\
-\infty &\text{ if } (x,y)\in(\fpq\cup\{-\infty\})\times\{-\infty\},\\
\nan &\text{ otherwise},
\end{cases}\\
x\otimes y=y\otimes x&\defeq\begin{cases}
\round{x\times y} &\text{ if } (x,y)\in\fpq^2,\\
\infty &\text{ if } (x,y)\in\big((\fpq_-\cup\{-\infty\})\times\{-\infty\}\big)\cup\big((\fpq_+\cup\{\infty\})\times\{\infty\}\big),\\
-\infty &\text{ if } (x,y)\in\big((\fpq_-\cup\{-\infty\})\times\{\infty\}\big)\cup\big((\fpq_+\cup\{-\infty\})\times\{\infty\}\big),\\ 
\nan\!\!\!\!\! &\text{ otherwise},
\end{cases}\\
x\ominus y&\defeq x\oplus(-1\otimes y),\\
x\oslash y&\defeq\begin{cases}
\round{x/y} &\text{ if } 0\le x\le y<\infty\text{ and }y>0,\\
\nan &\text{ if } y=\nan.
\end{cases}
\end{align*}

\newpage
\section{Technical lemmas}

\begin{lemma}\label{lem:same-sum1}
For any $x\in\fpq$ and $z\in\fpq\cup\{\infty\}$ with $x,z\ge0$, and for any $n,m\in\bbN\cup\{0\}$, it holds that 
$z\oplus\bigoplus_{i=1}^{3\times2^{p}-1+n}x=z\oplus\bigoplus_{i=1}^{3\times2^{p}-1+m}x$.
Furthermore, $z\oplus\bigoplus_{i=1}^{3\times2^{p}-1}x\in \{0,z, z^+, \pm\infty,\nan\}\cup(\{2^e:e\in\bbZ\}\cap\fpq)$ and $\bigoplus_{i=1}^{3\times2^{p}-1}x\in \{0,\pm\infty,\nan\}\cup(\{2^e:e\in\bbZ\}\cap\fpq)$.
\end{lemma}
\begin{proof}
Without loss of generality, suppose that $z<\infty$ and $x\ne0$. Choose $e\in\bbZ$ such that $2^e\le x<2^{e+1}$.
We consider all possible cases.

{\bf Case $e\ge \emax-p-1$ and $x>2^{\emax-p-1}$.} \\
In this case, it holds that
$$
z\oplus\bigoplus_{i=1}^{3\times2^{p}-1}x\ge \bigoplus_{i=1}^{3\times2^{p}-1}x\ge\bigoplus_{i=1}^{3\times2^{p}-1}(2^{\emax-p-1})^+=\infty,$$
where the last equality is from \cref{lem:max-distinguish}.
This implies that $z\oplus\bigoplus_{i=1}^{3\times2^{p}-1+n}x=\infty$ for all $n\in\bbN\cup\{0\}$ and the result follows.

{\bf Case $e\ge \emax-p-1$ and $x=2^{\emax-p-1}$.} \\
Since $$z\oplus\bigoplus_{i=1}^{2^{p+1}}x\ge\bigoplus_{i=1}^{2^{p+1}}x=2^{\emax},$$
and $(2^{\emax})^+-2^{\emax}=2^{\emax-p}=2x$, one can observe that 
$z\oplus\bigoplus_{i=1}^{2^{p+1}+n}x$ is either $z$, $z^+$ (when $z\ge 2^{\emax}$) or $2^{\emax}$ (when $z<2^{\emax}$) for all $n\in\bbN\cup\{0\}$. Hence, the result follows.

{\bf Case $e\le \emax-p-2$ and $z\ge2^{e+p+2}$.} \\
In this case, since $2^e\le x<2^{e+1}$ and $z^+-z\ge(2^{e+\mbit+2})^+-2^{e+\mbit+2}=2^{e+2}$, one can observe that $z\oplus\bigoplus_{i=1}^kx=z$ for all $k\in\bbN$. Hence, the result follows. %

{\bf Case $e\le \emax-p-2$,  $z<2^{e+p+2}$, and $x=2^e$.} \\
If $z\ge2^{e+p+1}$, then $z\oplus\bigoplus_{i=1}^kx=z$ (or $z^+$) for all $k\in\bbN$ and the result follows. %
For $z<2^{e+p+1}$, it holds that
$$z\oplus\bigoplus_{i=1}^{2^\mbit+1}x\ge\bigoplus_{i=1}^{2^\mbit+1}x=\bigoplus_{i=1}^{2^\mbit+1}2^e=2^{e+\mbit+1}~~\text{and}~~z\oplus\bigoplus_{i=1}^{k}x\le2^{e+\mbit+1}$$
for all $k\in\bbN$.
Hence, one can observe that $z\oplus\bigoplus_{i=1}^{2^\mbit+1+n}x=2^{e+\mbit+1}$ for all $n\in\bbN\cup\{0\}$ and the result follows.

{\bf Case $e\le \emax-p-2$,  $z<2^{e+p+2}$, and $x>2^e$.} \\
Since $2^e< x<2^{e+1}$ and $(2^{e+\mbit+2})^+-2^{e+\mbit+2}=2^{e+2}$, 
one can observe that 
\begin{align}
z\oplus\bigoplus_{i=1}^kx\le 2^{e+\mbit+2}\label{eq:same-sum1-1}
\end{align}
for all $k\in\bbN$.
Furthermore, by \cref{lem:max-distinguish},
\begin{align}
z\oplus\bigoplus_{i=1}^{3\times2^\mbit-1}x\ge\bigoplus_{i=1}^{3\times2^\mbit-1}x\ge\bigoplus_{i=1}^{3\times2^\mbit-1}(2^e)^+=2^{e+\mbit+2}.\label{eq:same-sum1-2}
\end{align}
From \cref{eq:same-sum1-1,eq:same-sum1-2}, the result follows. This completes the proof.
\end{proof}

\begin{lemma}\label{lem:one-plus}
For any $x\in(1,2]_{\fpq}$, there exist $y\in(2^{-1},1]_{\fpq}$ and $z\in(2^{-1},1^+]_{\fpq}$ such that $x\otimes y=1^+$ and $x\otimes z=1^{++}$.
\end{lemma}
\begin{proof}
We first show the existence of $y$.
We assume $x\in(1^+,2)_{\fpq}$ since we can choose $y=1$ for $x=1^+$ and $y=(2^{-1})^+$ for $x=2$.
For $x\in(1^+,2)_{\fpq}$, let $\alpha_x$ be the smallest number in $[2^{-1},1]_{\fpq}$ such that $x\times \alpha_x>1$ (the exact multiplication here).
Then, $\alpha_x\in(2^{-1},1)_{\fpq}$ for all $x\in(1^+,2)_{\fpq}$ since $2^-\times2^{-1}<1$ and $1^{-}\times1^{++}>1$.
Furthermore,
since $x\in(1^+,2)_{\fpq}$, $\alpha_x-\alpha_x^-=2^{-\mbit-1}$, and $x\times \alpha_x^-\le 1$ (note that $\alpha_x>2^{-1}$ but $\alpha_x$ is chosen from $[2^{-1},1]_{\fpq}$),
it holds that
$$x\times \alpha_x=x\times \alpha_x^-+x\times (\alpha_x-\alpha_x^-)=x\times \alpha_x^-+x\times2^{-p-1}\in[1,1^+],$$
i.e., $x\otimes \alpha_x\in\{1,1^+\}$.

If $x\otimes z_x=1^+$, then choose $y=\alpha_x$.
Hence, we now assume $x\otimes \alpha_x=1$, i.e., $x\times \alpha_x\in(1,1+2^{-p-1}]$.
In this case, since $\alpha_x<1$ and $x\in(1^+,2)_{\fpq}$, we have
$$x\times \alpha_x^+=x\times \alpha_x+x\times(\alpha_x^+- \alpha_x)=x\times \alpha_x+x\times2^{-p-1}\in(1+2^{-p-1},1+3\times 2^{-p-1}).$$
This implies $x\otimes \alpha_x^+=1^+$, and choosing $y=\alpha_x^+$ shows the existence of the desired $y$ for all $x\in(1,2]_{\fpq}$.

We now show the existence of $z$. Again, without loss of generality, we assume $x\in(1^{++},2]_{\fpq}$; we can choose $z=1^+$ for $x=1^+$ and $z=1$ for $x=1^{++}$.
For each $x\in(1^{++},2]_{\fpq}$, choose $y_x\in(2^{-1},1]_{\fpq}$ such that $x\otimes y_x=1^+$. Here, the existence of such $y_x$ is given by the first part of the proof. 
Then, one can observe that $y_x\in(2^{-1},1^-)_{\fpq}$ since $1^{+++}\otimes 1^-=1^{++}$.
Furthermore, since $x\otimes y_x=1^+$, it holds that $x\times y_x\in(1+2^{-p-1},1+3\times2^{-p-1})$.
Since $x\in(1^{++},2]_{\fpq}$ and $y_x^+-y_x=y_x^{++}-y_x^+=2^{-p-1}$ by $y_x\in(2^{-1},1^-)_{\fpq}$, one can observe that 
$$x\times y_x^+\in[1+3\times2^{-p-1},1+5\times2^{-p-1}]~~\text{or}~~x\times y_x^{++}\in[1+3\times2^{-p-1},1+5\times2^{-p-1}].$$
This implies that $x\otimes y_x^+=1^{++}$ or $x\otimes y_x^{++}=1^{++}$. Since $y_x^+,y_x^{++}\in(2^{-1},1]_{\fpq}$, choosing $z$ corresponding $y_x^+$ or $y_x^{++}$ completes the proof.
\end{proof}

\begin{lemma}\label{lem:distinguish}
For any $d_0\in\bbN$. There exist $d_1,d_2\in\bbN$ and floating-point affine transformations $\phi_1:\efpq^{d_0}\to\efpq^{d_1}$, $\phi_2:\efpq^{d_1}\to\efpq^{d_2}$ such that $f:\fpq^{d_0}\to\fpq^{d_2}$ defined as
$$f(\bfx)=\rho\circ\phi_2\circ\rho\circ\phi_1(\bfx)$$
is injective and $f(\fpq^{d_0})\subset[0,2^{\emax}]_{\fpq}^{d_2}$.
Furthermore, all values that appear during the intermediate computation are finite.
\end{lemma}
\begin{proof}
Let $d_2=6d_0$. Choose $\phi_1,\phi_2$ and $d_1$ so that %
\begin{align*}
\big(\rho\circ\phi_2\circ\rho\circ\phi_1(\bfx)\big)_{6i-5}&=\rho(1\otimes\rho(1\otimes x_i)\ominus2\otimes\rho(2^{-1}\otimes x_i)),\\
\big(\rho\circ\phi_2\circ\rho\circ\phi_1(\bfx)\big)_{6i-4}&=\rho(2\otimes\rho(2^{-1}\otimes x_i)\ominus1\otimes\rho(1\otimes x_i)),\\
\big(\rho\circ\phi_2\circ\rho\circ\phi_1(\bfx)\big)_{6i-3}&=\rho(1\otimes\rho(2^{-1}\otimes x_i)),\\
\big(\rho\circ\phi_2\circ\rho\circ\phi_1(\bfx)\big)_{6i-2}&=\rho(1\otimes\rho(-1\otimes x_i)\ominus2\otimes\rho(-2^{-1}\otimes x_i)),\\
\big(\rho\circ\phi_2\circ\rho\circ\phi_1(\bfx)\big)_{6i-1}&=\rho(2\otimes\rho(-2^{-1}\otimes x_i)\ominus1\otimes\rho(-1\otimes x_i)),\\
\big(\rho\circ\phi_2\circ\rho\circ\phi_1(\bfx)\big)_{6i}&=\rho(1\otimes\rho(-2^{-1}\otimes x_i)),
\end{align*}
for all $i\in [d_0]$ and $\bfx=(x_1,\dots,x_{d_0})\in\fpq^{d_0}$.
Then, for $f(\bfx)=\rho\circ\phi_2\circ\rho\circ\phi_1(\bfx)$, we have $f(\fpq^{d_0})\subset[0,2^{\emax}]$.
Here, if $x_i\ge2^{\emin+1}$, %
then $f(\bfx)_{6i-5}=0$, $f(\bfx)_{6i-4}=0$, and $f(\bfx)_{6i-3}=x_i/2>2^{\emin}$.
In addition, if $0\le x_i=m_0.m_1,\dots,m_\mbit \times 2^{\emin}\le2^{\emin+1}$ for some $m_0,\dots,m_\mbit\in\{0,1\}$, then
\begin{align*}
f(\bfx)_{6i-5}&=\begin{cases}
0~&\text{if}~m_p=0,\\
\omega~&\text{if}~m_{p-1}=0,m_p=1,\\
0 ~&\text{if}~m_{p-1}=1,m_p=1,
\end{cases}\\ 
f(\bfx)_{6i-4}&=\begin{cases}
0~&\text{if}~m_p=0,\\
0~&\text{if}~m_{p-1}=0,m_p=1,\\
\omega~&\text{if}~m_{p-1}=1,m_p=1,
\end{cases}\\
f(\bfx)_{6i-3}&=\begin{cases}
x_i/2~&\text{if}~m_p=0,\\
(x_i-\omega)/2~&\text{if}~m_{p-1}=0,m_p=1,\\
(x_i+\omega)/2~&\text{if}~m_{p-1}=1,m_p=1.
\end{cases} 
\end{align*}
and $f(\bfx)_{6i-3}\le2^{\emin}$.
We can also show a similar result for $f(\bfx)_{6i-2},f(\bfx)_{6i-1},f(\bfx)_{6i}$ when $x_i<0$.
This shows the injectivity of $f$ on $\fpq^{d_0}$ and completes the proof.
\end{proof}

\begin{lemma}[Modification of Lemma 4.1 in \cite{hwang2025floating}]\label{lem:indc}
Let $d_0,d_1\in\bbN$, $\bfz\in\fpq^{d_0}$, and $f:\fpq^{d_0}\to[-2^{\emax},2^{\emax}]_{\fpq}$ be an injective function.
Then, there exist $d_2,d_3\in\bbN$ and floating-point affine transformations %
$\phi_1:\efpq^{d_1}\to\efpq^{d_2}$ and $\phi_2:\efpq^{d_2}\to\efpq^{d_3}$ such that for any $\bfx\in\fpq^{d_0}$,
$$f(\bfx)=\rho\circ\phi_2\circ\rho\circ\phi_1\circ f(\bfx)=\indcc{\bfx=\bfz}.$$
Furthermore, all values that appear during the intermediate computation are finite.
\end{lemma}
\begin{proof}
Since the proof is almost identical to that of Lemma 4.1 in \cite{hwang2025floating}, we omit the proof.
\end{proof}

\begin{lemma}\label{lem:fnn-univ-approx}
Let $d_0,d_6\in\bbN$ and $f^*:\fpq^{d_0}\to\fpq^{d_6}$.
Then, there exist $d_1,d_2,d_3,d_4,d_5\in\bbN$ and floating-point affine transformations $\phi_i:\efpq^{d_{i-1}}\to\efpq^{d_{i}}$ for all $i\in[6]$ such that
$$f(\bfx)=\phi_6\circ\rho\circ\phi_5\circ\rho\circ\phi_4\circ\rho\circ\phi_3\circ\rho\circ\phi_2\circ\rho\circ\phi_1(\bfx)=f^*(\bfx)$$
for all $\bfx\in\fpq^{d_0}$.
Furthermore, all values that appear during the intermediate computation are finite.
\end{lemma}
\begin{proof}
Let $d_5=|\fpq|^{d_0}$ and $\bfz_1,\dots,\bfz_{d_5}$ be all elements in $\fpq^{d_0}$.
Then, by \cref{lem:distinguish,lem:indc}, there exists $d_1,d_2,d_3,d_4\in\bbN$ and floating-point affine transformations $\phi_j:\efpq^{d_{j-1}}\to\efpq^{d_{j}}$ for all $j\in[5]$ such that 
$$\big(\rho\circ\phi_5\circ\rho\circ\phi_4\circ\rho\circ\phi_3\circ\rho\circ\phi_2\circ\rho\circ\phi_1(\bfx)\big)_i=\indcc{\bfx=\bfz_i},$$
for all $i\in[d_5]$.
Choose $W=[\bfw_1,\dots,\bfw_{d_5}]\in\fpq^{d_6\times d_5}$ such that $\bfw_i=f^*(\bfz_i)$ for all $i\in[d_5]$.
Then, one can observe that 
$$W\otimes\big(\rho\circ\phi_5\circ\rho\circ\phi_4\circ\rho\circ\phi_3\circ\rho\circ\phi_2\circ\rho\circ\phi_1(\bfx)\big)=f^*(\bfx)$$
for all $\bfx\in\fpq^{d_0}$. This completes the proof.
\end{proof}
\begin{lemma}\label{lem:narrow-fnn}
Let $k\in\bbN$, $d_0,\dots,d_k\in\bbN$ with $d_k=1$, $\phi_i:\efpq^{d_{i-1}}\to\efpq^{d_{i}}$ be floating-point affine transformations, and $w\in\fpq$ such that all values that appear during the intermediate computation of  
$$f(\bfx)=w\otimes(\rho\circ\phi_k\circ\rho\circ\cdots\circ\rho\circ\phi_1(\bfx))$$
are finite. 
Let $m,r,h=1$, $d=d_0+k$.
Then, there exists a transformer block $g$ with $m,r,h,d$ such that
$$g(X)_i=\begin{cases}
X_i&\text{if}~i\in[d_0],\\
X_i\oplus[f(\bfx_1),\dots,f(\bfx_n)]&\text{if}~i=d,\\
\zerov_n^\top&\text{otherwise}.
\end{cases}$$

\end{lemma}
\begin{proof}
Throughout this proof, we will consider $g$ as a composition of feed-forward networks %
with $r=1$ and $d=d_0+k$, and consider attention layers as identity maps.
We first note that composition of feed-forward networks with $r=1$ can perform $(x_1,\dots,x_d)\mapsto(x_1,\dots,x_{i-1},x_i\oplus x_j,x_{i+1},\dots,x_d)$ or $(x_1,\dots,x_d)\mapsto(x_1,\dots,x_{i-1},x_i\ominus x_j,x_{i+1},\dots,x_d)$ if $x_1,\dots,x_d\in\fpq$.
This can be done by sequentially performing $(x_1,\dots,x_d)\mapsto(x_1,\dots,x_{i-1},x_i\oplus\rho(x_j),x_{i+1},\dots,x_d)$ and $(x_1,\dots,x_d)\mapsto(x_1,\dots,x_{i-1},x_i\ominus\rho(-x_j),x_{i+1},\dots,x_d)$, or $(x_1,\dots,x_d)\mapsto(x_1,\dots,x_{i-1},x_i\ominus\rho(x_j),x_{i+1},\dots,x_d)$ and $(x_1,\dots,x_d)\mapsto(x_1,\dots,x_{i-1},x_i\oplus\rho(-x_j),x_{i+1},\dots,x_d)$.
Using this, we can zero out all dimensions that are not in $[d_0]\cup\{d\}$.
Likewise, we can add/subtract some constant from some coordinate.
Hence, without loss of generality, we consider $X_i=\zerov_n^\top$ for all $i\notin[d_0]\cup\{d\}$ for the input.

We use mathematical induction on $k$ on the statement of the lemma. %
One can observe that the base case ($k=0, d = d_0$) is trivial: consider a feed-forward network $\bfx=(x_1,\dots,x_d)\mapsto(x_1,\dots,x_{d_0},x_{d_0}\oplus w\otimes(\rho\circ\phi_1(\bfx)))$.
Now, consider general $k\ge1$ and let 
$\phi_k(z_1,\dots,z_{d_{k-1}})=(\bigoplus_{j=1}^{d_{k-1}} w_j\otimes z_j)\oplus b$ for some $w_{j},b\in\fpq$.
Define $\psi_1,\dots,\psi_{d_{k-1}} :\fpq^{d_0}\to\fpq$ as
$$\psi_j(\bfx)=w_j\otimes(\rho\circ(\phi_{k-1}\circ\rho\circ\cdots\circ\rho\circ\phi_1(\bfx))_j), \quad j=1, \dots, d_{k-1},$$
Then, by the induction hypothesis, there exists a transformer block $g'_j$ with $m,r,h=1$ and $d=d_0+k$ such that
$$g'_j(X)_i=\begin{cases}
X_i&\text{if}~i\in[d_0]\cup\{d\},\\
X_i \oplus[\psi_{j}(\bfx_1),\dots,\psi_{j}(\bfx_n)]&\text{if}~i=d-1,\\
\zerov_n^\top&\text{otherwise},
\end{cases}$$
where $[\bfx_1,\dots,\bfx_n]=X_{:d_0}$.
Let $g'_{d_{k-1}+1}(X)_i$ be a transformer block that adds $b\in\fpq$ to the $i$-th row of $X$ (this requires $r=1$).
Then, one can observe that $h(X)=g'_{d_{k-1}+1}\circ\cdots\circ g'_{1}(X)$
satisfies
$$h(X)_i=\begin{cases}
X_i&\text{if}~i\in[d_0]\cup\{d\},\\
[\phi_k\circ\rho\circ\cdots\rho\circ\phi_1(\bfx_1),\dots,\phi_k\circ\rho\circ\cdots\rho\circ\phi_1(\bfx_n)]&\text{if}~i=d-1,\\
\zerov_n^\top&\text{otherwise}.
\end{cases}$$
We apply the following maps to the output of $h$: $(x_1,\dots,x_d)\mapsto(x_1,\dots,x_{d-1},x_d\oplus w\otimes\rho(x_{d-1}))$, $(x_1,\dots,x_d)\mapsto(x_1,\dots,x_{d-1},x_d\ominus w\otimes\rho(-x_{d-1}))$.
Then, we zero-out the $(d-1)$-th dimension. This completes the proof.
\end{proof}

\begin{lemma}\label{lem:oneppp}
Suppose $\mbit \ge 3$. Then, it holds that 
\begin{align*}
    (1^+\oplus1^{++})\oplus1^{++}&=(1^{++}\oplus1^{+})\oplus1^{++}=3^{+++}, \\
    (1^{++}\oplus1^{++})\oplus1^{+}&=3^{++}.
\end{align*}
  
\end{lemma}
\begin{proof}
We have 
    \begin{align*}
        (1^+\oplus1^{++})\oplus1^{++} &= \round{2+3\times 2^{-\mbit} } \oplus (1 + 2^{1-\mbit}) = (2 + 2^{2-\mbit}) \oplus (1 + 2^{1-\mbit}) = 3 + 3\times 2^{1-\mbit} = 3^{+++}. \\
        (1^{++}\oplus1^{++})\oplus1^{+}  &= (2 + 2^{2-\mbit}) \oplus (1 + 2^{-\mbit}) = \round{3 + 5\times 2^{-\mbit}} = 3 + 2^{2-\mbit} = 3^{++}.
    \end{align*}
If $\mbit = 2$, we have 
    \begin{align*}
        (1^+\oplus1^{++})\oplus1^{++} &= (1.01 \oplus 1.10) \oplus 1.10 = (1.10 \times 2) \oplus 1.10 =  1.00 \times 2^2 = 3^{++}. \\
        (1^{++}\oplus1^{++})\oplus1^{+}  &= (1.10 \oplus 1.10) )\oplus1.01 = (1.10 \times 2)  \oplus1.01 =  11.0 + 1.01 = 1.00 \times 2^2 = 3^{++}.
    \end{align*}
\end{proof}

\begin{lemma}\label{lem:onep2}
Suppose $\mbit=2$. Then, it holds that
\begin{align*}
(1^+\oplus1)\oplus1^+&=(1\oplus1^+)\oplus1^+=3,\\
(1^+\oplus1^+)\oplus1&=3^+.
\end{align*}
Furthermore, for any $x\in(1,2]_{\fpq}$, there exist $y\in [2^{-1},1)_{\fpq}$ such that $x\otimes y=1$.
\end{lemma}
\begin{proof}
For the first statement, we have
\begin{align*}
(1^+\oplus1)\oplus1^+&=(1\oplus1^+)\oplus1^+=\round{2+2^{-2}}\oplus(1+2^{-2})=2\oplus(1+2^{-2})=3,\\
(1^+\oplus1^+)\oplus1&=\round{2+2^{-1}}\oplus1=(2+2^{-1})\oplus1=3^+.
\end{align*}
For the second statement, if $x \in (1,2]_\fpq$, then we have $x \in \{ 1+2^{-2},  1+2^{-1}, 1+2^{-1}+2^{-2} , 2\}$. Since
\begin{align*}
(1+2^{-2})\otimes((1+2^{-1}))&=\round{1+2^{-1}+2^{-2}+2^{-3}}=2,\\
(1+2^{-2})\otimes((1+2^{-1}+2^{-2}))&=\round{1+2^{-1}+2\times2^{-2}+2^{-3}+2^{-4}}=2,\\
2^{-1}\otimes2 &=1,
\end{align*}
for  $x = 1+2^{-2},  1+2^{-1}, 1+2^{-1}+2^{-2} ,   2 $, we have $y_x = (2^{-1}+2^{-2}), \;  2^{-1}+2^{-3}, \; 2^{-1}+2^{-3}, 2^{-1} \in [2^{-1}, 1)_\fpq $ respectively. 
This completes the proof.
\end{proof}

\newpage

\section{Proofs}
\subsection{Proof of \cref{thm:equiv}}\label{sec:pfthm:equiv}

Since the feed-forward networks are column-wise applications, they are permutation equivariant (i.e., $\pi_{(1,2)}^n$ equivariant).
Furthermore, one can observe that 
for $W\in\fpq^{m\times d}$, $Z\in\fpq^{d\times n}$, and permutation equivariant $\psi':\fpq^{d\times n}\to\fpq^{d\times n}$, $\phi(Z)=W\otimes Z$ and $\psi(Z)=Z+\psi'(Z)$ are also permutation equivariant.
Hence, it suffices to show that
$g(Z)=(W^V\otimes Z)\otimes\sigma((W^K\otimes Z)^\top(W^Q\otimes Z))$ is $\pi_{(1,2)}^n$-equivariant for all $W^V,W^K,W^Q\in\fpq^{m\times d}$.

Let $Z\in\fpq^{d\times n}$ and $Z'=\pi_{(1,2)}^n Z$.
Let $V=[\bfv_1,\dots,\bfv_n]=W^V\otimes Z$, $V'=W^V\otimes Z'$, $S=\sigma((W^K\otimes Z)^\top(W^Q\otimes Z))$, and $S'=\sigma((W^K\otimes Z')^\top(W^Q\otimes Z'))$.
Then, $V'=\pi_{(1,2)}V$.
Furthermore, by the definition of the floating-point softmax $\sigma$, we have
\begin{align*}
S'=\begin{bmatrix}
S_{22} & S_{21} & S_{23} & \dots & S_{2n}\\
S_{12} & S_{11} & S_{13} & \dots & S_{1n}\\
S_{32} & S_{31} & S_{33} & \dots & S_{3n}\\
\vdots & \vdots & \vdots & \ddots &\vdots\\
S_{n2} & S_{n1} & S_{n3} & \dots & S_{nn}
\end{bmatrix}, \quad \text{where} \quad S= \begin{bmatrix}
S_{11} & S_{12} & S_{13} & \dots & S_{1n}\\
S_{21} & S_{22} & S_{23} & \dots & S_{2n}\\
S_{31} & S_{32} & S_{33} & \dots & S_{3n}\\
\vdots & \vdots & \vdots & \ddots &\vdots\\
S_{n1} & S_{n2} & S_{n3} & \dots & S_{nn}
\end{bmatrix}.
\end{align*}
By the definition of $g$, it holds that $g(Z)_{ij}=\bigoplus_{k=1}^n(V_{ik}\otimes S_{kj})$, and $g(Z')_{ij}=\bigoplus_{k=1}^n(V'_{ik}\otimes S'_{kj})$. Furthermore, for $j\ge3$, we have
\begin{align*}
g(Z')_{i1}&=(V_{i2}\otimes S_{22})\oplus(V_{i1}\otimes S_{12})\oplus\bigoplus_{k=3}^n(V_{ik}\otimes S_{k2})=\bigoplus_{k=1}^n(V_{ik}\otimes S_{k2})=g(Z)_{i2},\\
g(Z')_{i2}&=(V_{i2}\otimes S_{21})\oplus(V_{i1}\otimes S_{11})\oplus\bigoplus_{k=3}^n(V_{ik}\otimes S_{k1})=\bigoplus_{k=1}^n(V_{ik}\otimes S_{k1})=g(Z)_{i1},\\
g(Z')_{ij}&=(V_{i2}\otimes S_{2j})\oplus(V_{i1}\otimes S_{1j})\oplus\bigoplus_{k=3}^n(V_{ik}\otimes S_{kj})=\bigoplus_{k=1}^n(V_{ik}\otimes S_{kj})=g(Z)_{ij}.
\end{align*}
Hence, $g(Z')=\pi_{(1,2)}^ng(Z)$. This completes the proof.

\subsection{Proof of \cref{thm:eq-preserve}}\label{sec:pfthm:eq-preserve}
Since feed-forward networks are token-wise applications in floating-point transformers, they preserve equality.
Furthermore, 
for $W\in\fpq^{m\times d}$, $Z\in\fpq^{d\times n}$, and equality preserving $\psi':\fpq^{d\times n}\to\fpq^{d\times n}$, $\phi(Z)=W\otimes Z$ and $\psi(Z)=Z+\psi'(Z)$ also preserve equality.
Lastly, one can observe that for $W^K,W^Q\in\fpq^{m\times d}$ and $Z=[\bfz_1,\dots,\bfz_n]\in\fpq^{d\times n}$,
the $\sigma((W^K\otimes Z)^\top(W^Q\otimes Z))=[\bfs_1,\dots,\bfs_n]$ satisfies $\bfs_i=\bfs_j$ when $\bfz_i=\bfz_j$, i.e., $\sigma((W^K\otimes Z)^\top(W^Q\otimes Z))$ preserves equality.
Hence, all floating-point transformers preserve equality. This completes the proof.

\subsection{Proof of \cref{lem:pos-enc}}\label{sec:pflem:pos-enc}
Since the proof is trivial if $z\in\{\pm\infty,\nan\}$,
without loss of generality, suppose $z>0$.
If $z\ge2^{\emin+1}$, then $0\oplus z=\omega\oplus z=z$.
For $z<2^{\emin+1}$, let $x$ be the smallest positive float such that $x\oplus z=x$; then, $x\le2^{\emin+\mbit+2}\le\Omega$.
Since $0\oplus z=z>0$, by the pigeonhole principle, there must be two distinct floats (say $a,b$) between zero and $x$ such that $a\oplus z=b\oplus z$. Hence, $x\mapsto x+z$ is not injective for all non-zero $z\in\fpq$.

\subsection{Proof of \cref{lem:diagonal-factorize}}\label{sec:pflem:diagonal-factorize}
For each $X=[\bfx_1,\dots,\bfx_n]\in\Delta_n$, let $[\bfy_{X,1},\dots,\bfy_{X,n}] =f^*(X)$ and choose $f : \fpq^{\din}\times\fpq^{\din\times n}\to\fpq^{\dout}$ such that $f(\bfx_i,X)=\bfy_{X,i}$.
Then, $f$ is well-defined since $X\in\Delta_n$, i.e., for each $\bfx\in\{\bfx_1,\dots,\bfx_n\}$, there exists a unique $i$ such that $\bfx_i=\bfx$.
Furthermore, since $f^*$ is $\pi_{(1,2)}^n$-equivariant and $(\pi_{(1,2)}^n)^{-1}=\pi_{(1,2)}^n$, 
\begin{align*}
[f(\bfx_1,X),\dots,f(\bfx_n,X)]%
&=f^*(X)=(\pi_{(1,2)}^n)^{-1}f^*(\pi_{(1,2)}^nX)
=(\pi_{(1,2)}^n)^{-1}[\bfy_{X,2},\bfy_{X,1},\bfy_{X,3},\dots,\bfy_{X,n}]\\
&=(\pi_{(1,2)}^n)^{-1}[f(\bfx_2,\pi_{(1,2)}^nX),f(\bfx_1,\pi_{(1,2)}^nX),f(\bfx_3,\pi_{(1,2)}^nX),\dots,f(\bfx_n,\pi_{(1,2)}^nX)]\\
&=[f(\bfx_1,\pi_{(1,2)}^nX),\dots,f(\bfx_n,\pi_{(1,2)}^nX)]
\end{align*}
for all $X=[\bfx_1,\dots,\bfx_n]\in\Delta_n$.
This completes the proof.

\subsection{Proof of \cref{lem:three-max}}\label{sec:pflem:three-max}
Since the proof is trivial if $n=2$, we consider $n\ge3$.
Let $\bfx\in f(\mcS_n)$. 
Since $\bfx\in f(\mcS_n)$, there exists $\pi\in\mcS_n$ such that $f(\pi)=\bfx$.
In addition, by the definition of $f$, one can observe that $\bfx=f(\pi)=f(\pi_{(1,2)}^n\circ\pi)$, i.e., $\{\pi,\pi_{(1,2)}^n \circ \pi\}\subset f^{-1}(\bfx)$.
We now show that for any $\phi\notin\{\pi,\pi_{(1,2)}^n\circ \pi\}$, $\phi\notin f^{-1}(\bfx)$.
Let $\phi \in \mcS_n \setminus \{\pi,\pi_{(1,2)}^n \circ \pi \} $.
Then, there exists $i\in[n]$ such that $\pi^{-1}(i)\ne\phi^{-1}(i)$. Choose the largest such $i$ and denote it as $i^*$, i.e., $\phi^{-1}(j)=\pi^{-1}(j)$ for all $j\in[n]\setminus[i^*]$, and let $i_1=\pi^{-1}(i^*)$, $i_2=\phi^{-1}(i^*)$. %
Note that $i^*\ge3$: $i^*$ cannot be one, and if $i^*=2$, then $\phi=\pi_{(1,2)}^n\circ\pi$.
Since $$| \{j \in [n]:\pi(j)>i^*\}| \le n -i^* \le n-3,$$$[n]\setminus(\{i_1,i_2\}\cup\{j \in [n]:\pi(j)>i^*\})$ is nonempty. Let $i_3\in[n]\setminus(\{i_1,i_2\}\cup\{j \in [n]:\pi(j)>i^*\})$. 
Then, we have $\pi(i_2)<\pi(i_1)=i^*$; if $\pi(i_2)>\pi(i_1)=i^*$, then $i_2=\pi^{-1}(\pi(i_2))\ne\phi^{-1}(\pi(i_2))$, which contradicts the definition of $i^*$.
Likewise, we have $\phi(i_2)>\phi(i_1)$.
Furthermore, $\pi(i_3)<\pi(i_1)=i^*$ by the definition of $i_3$, which implies that $\phi(i_3)<\phi(i_2)=i^*$; otherwise, $\pi^{-1}(\phi(i_3))\ne\phi^{-1}(\phi(i_3))=i_3$ since $\pi(i_3)<i^*$, which contradicts the definition of $i^*$.
Hence, it holds that
$$\argmax_{i\in\{i_1,i_2,i_3\}}\pi(i)=i_1\ne i_2=\argmax_{i\in\{i_1,i_2,i_3\}}\phi(i),$$
i.e., $\bfx=f(\pi)\ne f(\phi)$.
This completes the proof.

\subsection{Proof of \cref{lem:token-univ-approx}}\label{sec:pflem:token-univ-approx}

For notational simplicity, we use $t_d$ and $t_r$ to denote $t_d'$ and $t_r'$.
Since all attention layers can represent the identity map for all $m,d,h\in\bbN$ by choosing zero weights, we only consider feed-forward networks in this proof.
By \cref{lem:fnn-univ-approx}, there exist $d_1=e_0,e_1,e_2,e_3,e_4,e_5,e_6=d_2\in\bbN$ and affine transformations $\phi_{i-1}:\efpq^{e_{i-1}}\to\efpq^{e_i}$ for all $i\in[6]$ such that for each $\bfx\in\fpq^{e_0}$,
$$\phi_6\circ\rho\circ\phi_5\circ\rho\circ\phi_4\circ\rho\circ\phi_3\circ\rho\circ\phi_2\circ\rho\circ\phi_1(\bfx)=\psi^*(\bfx).$$
Choose $t_d=e_6+\max\{e_6,\sum_{i=0}^5e_i\}$ and $t_r=2\times\max\{e_6,\sum_{i=0}^5e_i\}$.
Then, it suffices to show that there exist $m\in\bbN$ and $\psi_1,\dots,\psi_m:\fpq^{t_d}\to\fpq^{t_d}$ such that 
$$\psi_i(\bfz)=\bfz\oplus(g_{i,2}\circ \rho\circ g_{i,1}(\bfz))$$
and %
$$\psi([\bfx_1,\dots,\bfx_n])=[\psi_m\circ\cdots\circ\psi_1(\bfx_1),\dots,\psi_m\circ\cdots\circ\psi_1(\bfx_n)]$$ satisfies the desired property of $\psi$
for some floating-point affine transformations $g_{i,1}:\efpq^{t_d}\to\efpq^{t_r}$ and $g_{i,1}:\efpq^{t_r}\to\efpq^{t_d}$ for all $i\in[k]$. Here, we assume $d=t_d$ and $r=t_r$. 
For $d>t_d$ and $r>t_r$, we pad unused dimensions/neurons.

Let $\alpha_i=\sum_{j=0}^{i-2}e_i$ for $i\in[8]$.
We construct $\psi_1,\dots,\psi_6$ as follows: for $\bfz=(z_1,\dots,z_{t_d})\in\fpq^{t_d}$, $\bfy=(y_1,\dots,y_{t_r})\in\fpq^{t_r}$, $j\in[t_r]$, and $k\in[t_d]$,
\begin{align*}
g_{i,1}(\bfz)_{j}&=\begin{cases}
\phi_i(z_{\alpha_i+1},\dots,z_{\alpha_i+e_{i-1}})_j~&\text{if}~j\le e_i,\\
0~&\text{if}~j>e_i,
\end{cases}\\
g_{i,2}(\bfy)_k&=\begin{cases}
y_{k-\alpha_{i+1}}~&\text{if}~\alpha_{i+1}<k\le\alpha_{i+2},\\
0~&\text{otherwise},
\end{cases}
\end{align*}
for all $i\in[5]$.
We also define $\psi_6$ as
\begin{align*}
g_{6,1}(\bfz)_{j}&=\begin{cases}
\phi_6(z_{\alpha_6+1},\dots,z_{\alpha_6+e_{i-1}})_j~&\text{if}~j\le e_6,\\
0~&\text{if}~j>e_6,
\end{cases}\\
g_{6,2}(\bfy)_k&=\begin{cases}
y_{k-\alpha_{7}}~&\text{if}~t_d-e_6<k\le t_d,\\
0~&\text{otherwise}.
\end{cases}
\end{align*}
Here, note that $t_d-e_6\ge\alpha_7$ by the definition of $t_d$.
Then, one can observe that for any $\bfz=(z_1,\dots,z_{t_d})\in\fpq^{t_d}$ such that $z_{e_0+1},\dots,z_{t_d}=0$, it holds that
$\psi_6\circ\cdots\circ\psi_1(\bfz)\in\fpq^{t_d}.$
Furthermore, for $j\in[e_6]$ and $k\in[\max\{0,e_6-\alpha_7\}]$, we have
\begin{align*}
\psi_6\circ\cdots\circ\psi_1(\bfz)_{t_d-e_6+j}&=\psi^*(z_1,\dots,z_{e_0})_j,\\
\psi_6\circ\cdots\circ\psi_1(\bfz)_{\alpha_7+k}&=0.
\end{align*}

We next erase the first $\alpha_{7}$ coordinates of the output of $\psi_6\circ\cdots\circ\psi_1$.
Define $\psi_7$ as follows: for $j\in[t_d]$,
\begin{align*}
g_{7,2}\circ\rho\circ g_{7,1}(z_1,\dots,z_{t_d})_j=\begin{cases}
(-1)\otimes\rho(1\otimes z_j)\oplus1\otimes\rho((-1)\otimes z_j)~&\text{if}~j\le\alpha_7,\\
0~&\text{if}~j>\alpha_7.
\end{cases}
\end{align*}
Then,
since all intermediate computation of the network in \cref{lem:fnn-univ-approx} are finite, it holds that
\begin{align*}
\psi_7\circ\cdots\circ\psi_1(\bfz)_j=\begin{cases}
\psi^*(z_1,\dots,z_{e_0})_{j-(t_d-e_6)}~&\text{if}~j>t_d-e_6,\\
0~&\text{if}~j\le t_d-e_6,
\end{cases}
\end{align*}
for all $\bfz=(z_1,\dots,z_{t_d})\in\fpq^{t_d}$ with $z_{e_0+1},\dots,z_{t_d}=0$.
We lastly define $\psi_8$ as
\begin{align*}
g_{8,2}\circ\rho\circ g_{8,1}(z_1,\dots,z_{t_d})_j=\begin{cases}
1\otimes\rho(1\otimes z_{j+(t_d-e_6)})\oplus(-1)\otimes\rho((-1)\otimes z_{j+(t_d-e_6)})~&\text{if}~j\le e_6,\\
0~&\text{if}~e_6<j\le t_d-e_6,\\
(-1)\otimes\rho(1\otimes z_j)\oplus1\otimes\rho((-1)\otimes z_j)~&\text{if}~j>t_d-e_6,
\end{cases}
\end{align*}
for all $\bfz=(z_1,\dots,z_{t_d})\in\fpq^{t_d}$.
We choose $\psi=\psi_8\circ\cdots\circ\psi_1$.
Then, for any $\bfz=(z_1,\dots,z_{t_d})\in\fpq^{t_d}$ with $z_{e_0+1},\dots,z_{t_d}=0$, it holds that
$$\psi(\bfz)_j=\begin{cases}
\psi^*(z_1,\dots,z_{e_0})_j~&\text{if}~j\le e_6,\\
0~&\text{if}~j>e_6.    
\end{cases}$$
This completes the proof.

\subsection{Proof of \cref{lem:diagonal-attention}}\label{sec:pflem:diagonal-attention}
We first consider the case that $\mbit\ge3$.
Let $\gamma=\binom{|\fpq|^{\din}}3$, $1\oslash(\bigoplus_{i=1}^n1)=\alpha\ge2^{-\mbit-1}$, and $\beta,\beta'\in\fpq$ such that $\beta\otimes\alpha=1^+$ and $\beta'\otimes\alpha=1^{++}$.
Such $\beta,\beta'$ always exist by \cref{lem:one-plus}. 
Without loss of generality, we assume that $d=t_d,m=t_m,$ and $r=t_r$; otherwise, we can add unused dimensions. We set $t_m=3\gamma$, and we will choose the explicit values of $t_d,t_r$ later with $t_d \ge \din + 6 \gamma$.
We construct $\phi$ as a composition of two transformer blocks $g_1,g_2\in\mcB^{1,t_m,t_r,t_d,n}$. Here, $g_1$ maps $Z\in\fpq^{t_d\times n}$ with $Z_{:\din}=[\bfx_1,\dots,\bfx_n]\in\Delta_{\din}$ and $Z_{\din+1:}=\zerov_{t_d-\din,n}$ to $Y=[\bfy_1,\dots,\bfy_n]\in\fpq^{t_d\times n}$ such that
\begin{align*}
(\bfy_i)_l=\begin{cases}
(\bfx_i)_l~&\text{if}~l\le\din,\\
\beta'~&\text{if}~l=(t_d-3\gamma)+3j+k~\text{and}~\bfx_i\in\{\bfz_{j,\pi_k(1)},\bfz_{j,\pi_k(2)}\}~\text{for some}~j\in[\gamma-1]\cup\{0\},k\in[3],\\
\beta~&\text{if}~l=(t_d-3\gamma)+3j+k~\text{and}~\bfx_i=\bfz_{j,\pi_k(3)}~\text{for some}~j\in[\gamma-1]\cup\{0\},k\in[3],\\
0~&\text{otherwise},
\end{cases}
\end{align*}
for all $i\in[n]$ and $l\in[t_d]$.
By \cref{lem:fnn-univ-approx}, there exist $t_d',t_r'\in\bbN$ such that the desired $g_1\in\mcB^{1,m',r',d',n}$ exists for all $m'\ge1$, $d'\ge t_d'$, and $r'\ge t_r'$. Choose $t_d=\max\{t_d',\din+6\gamma\}$, $t_r=\max\{t_r',3\gamma\}$, and $g_1\in\mcB^{1,t_m,t_r,t_d,n}$.

We now construct $g_2$.
Define an attention layer $f_1:\fpq^{t_d\times n}\to\fpq^{t_d\times n}$ as 
\begin{align*}
f_1(Z)&=Z\oplus\left(W^O\otimes((W^V\otimes Z)\otimes\sigma((\zerov_{t_m,t_d}\otimes Z)^\top(\zerov_{t_m,t_d}\otimes Z)))\right)\\
&=Z\oplus\left(W^O\otimes((W^V\otimes Z)\otimes(\alpha\onev_{n,n}))\right),
\end{align*}
for some $W^V\in\fpq^{t_m\times t_d}$ and $W^O\in\fpq^{t_d\times t_m}$.
Specifically, we choose $W^V,W^O$ so that $W^V\otimes Z=Z_{t_d-3\gamma+1:}$, $W^O_{:\din}=\zerov_{\din,t_m}$, $W^O_{\din+1:\din+3\gamma}=I_{t_m}$, and $W^O_{\din+3\gamma+1:}=\zerov_{t_d-\din-3\gamma,t_m}$ where $I_{t_m}$ denotes the $t_m\times t_m$ identity matrix.
Then, by the definitions of $g_1,f_1$ and \cref{lem:oneppp}, it holds that for each $Z\in\fpq^{t_d\times n}$ with $Z_{:\din}=[\bfx_1,\dots,\bfx_n]\in\Delta_{\din}$ and $Z_{\din+1:}=\zerov_{t_d-\din,n}$, for each $j\in[\gamma-1]\cup\{0\}$ and $k\in[3]$,
\begin{itemize}
\item $(f_1\circ g_1(Z))_{:\din}=[\bfx_1,\dots,\bfx_n]$, 
\item $(f_1\circ g_1(Z))_{\din+3j+k}=3^{++}\!\times\onev_n^\top$ if there exists $i_1\!<i_2\!<i_3$ such that $\{\bfx_{i_1},\bfx_{i_2}\}=\{\bfz_{j,{\pi_k(1)}},\bfz_{j,{\pi_k(2)}}\}$ and $\bfx_{i_3}=\bfz_{j,{\pi_k(3)}}$,
\item %
$(f_1\circ g_1(Z))_{\din+3j+k}\in(\fpq\setminus\{3^{++}\})^{1\times n}$ otherwise,
\item $(f_1\circ g_1(Z))_{\din+3\gamma+1:t_d-3\gamma}=\zerov_{t_d-6\gamma-\din,n}$ and $(f_1\circ g_1(Z))_{t_d-3\gamma+1:}$ is non-negative.
\end{itemize}

Choose a feed-forward network $f_2:\fpq^{t_d\times n}\to\fpq^{t_d\times n}$ as $f_2(Z)=Z\oplus(W_2\otimes\rho(W_1\otimes Z))$ where $W_1\in\fpq^{t_r\times t_d}$ and $W_1\in\fpq^{t_d\times t_r}$ are matrices satisfying that for $\bfz=(z_1,\dots,z_{t_d})\in\fpq^{t_d}$ and $\bfy=(y_1,\dots,y_{t_r})\in\fpq^{t_r}$,
\begin{align*}
W_1\otimes\bfz=(z_{t_d-3\gamma+1},\dots,z_{t_d},0,\dots,0),~~
W_2\otimes\bfy=(0,\dots,0,-y_1,\dots,-y_{3\gamma}),
\end{align*}
Namely, $f_2$ erases the last $3\gamma$ coordinates of an input, if they are finite and non-negative.
Let $g_2=f_2\circ f_1$ and $\phi=g_2\circ g_1\in\mcB^{1,t_m,t_r,t_d,n}$. Then, by the definitions of $g_1$ and $g_2$, $\phi$ satisfies the desired properties. 

The proof for the case $\mbit=2$ is almost identical to the previous case $\mbit\ge3$. The only difference is our choices of $\beta,\beta'$, which we choose to satisfy 
$\beta\otimes\alpha=1$ and $\beta'\otimes\alpha=1^+$.
We note that the existence of such $\beta,\beta'$ is guaranteed by \cref{lem:one-plus,lem:onep2}.
Then, as in the previous case, we use the property that
\begin{align*}
(1^+\oplus1^+)\oplus1=3^+\ne3=(1^+\oplus1)\oplus1^+=(1\oplus1^+)\oplus1^+.
\end{align*}
This completes the proof.

\subsection{Proof of \cref{lem:permutation-factorize}}\label{sec:pflem:permutation-factorize}

For each $X=[\bfx_1,\dots,\bfx_n]\in\fpq^{\din\times n}$ and $\mcX=\{\bfx_1,\dots,\bfx_n\}$, define $f_X:\mcX\to\fpq^{\dout}$ as $f_X(\bfx)=\bfy_i$ for $i\in[n]$ satisfying $\bfx_i=\bfx$ where $[\bfy_1,\dots,\bfy_n]=f^*(X)$. Since $f^*$ is permutation equivariant, $f_X$ is well-defined.

Since $f^*$ is permutation equivariant, for any $\pi\in\mcS_n$, $f^*([\bfx_{\pi(1)},\dots,\bfx_{\pi(n)}])=[\bfy_{\pi(1)},\dots,\bfy_{\pi(n)}]$, i.e., $f_X=f_{\pi X}$ for all $\pi\in\mcS_n$.
Furthermore, since $\bfc_X=\bfc_{Z}$ if and only if $X=\pi Z$ for some $\pi\in\mcS_n$, we can define $\tilde f(\bfx_i,\bfc_X) \defeq f_X(\bfx_i)$ for all $X\in\fpq^{\din\times n}$; note that this is well-defined as $f_X=f_{\pi X}$ for all $\pi\in\mcS_n$.
Choose arbitrary values to $\tilde f(\bfx,\bfc)$ for all $\bfx,\bfc$ that we did not consider yet.
This completes the proof.

\subsection{Proof of \cref{lem:permutation-attention}}\label{sec:pflem:permutation-attention}

Let $\gamma=|\fpq|^{\din}$, $1\oslash(\bigoplus_{i=1}^n1)=\alpha\ge2^{-\mbit-1}$, and $\beta\in\fpq$ such that $\beta\otimes\alpha=1^+$.
Such $\beta$ always exists by \cref{lem:one-plus}. 
Without loss of generality, we assume that $d=t_d,m=t_m,$ and $r=t_r$; otherwise, we can add unused dimensions. We set $t_m=\gamma$, and we will choose the explicit values of $t_d,t_r$ later.
We construct $\phi$ as a composition of two transformer blocks $g_1,g_2\in\mcB^{1,t_m,t_r,t_d,n}$. Here, $g_1$ maps $Z\in\fpq^{t_d\times n}$ with $Z_{:\din}=[\bfx_1,\dots,\bfx_n]\in\fpq^{\din\times n}$ and $Z_{\din+1:}=\zerov_{t_d-\din,n}$ to $Y=[\bfy_1,\dots,\bfy_n]\in\fpq^{t_d\times n}$ such that
\begin{align*}
(\bfy_i)_l=\begin{cases}
(\bfx_i)_l~&\text{if}~l\le\din,\\
\beta~&\text{if}~l=(t_d-\gamma)+j~\text{and}~\bfx_i=\bfz_{j}~\text{for some}~j\in[\gamma],\\
0~&\text{otherwise},
\end{cases}
\end{align*}
for all $i\in[n]$ and $l\in[t_d]$.
By \cref{lem:fnn-univ-approx}, there exist $t_d',t_r'\in\bbN$ such that the desired $g_1\in\mcB^{1,m',r',d',n}$ exists for all $m'\ge1$, $d'\ge t_d'$, and $r\ge t_r'$. Choose $t_d=\max\{t_d',\din+2\gamma\}$, $t_r=\max\{t_r',\gamma\}$, and $g_1\in\mcB^{1,t_m,t_r,t_d,n}$.

We now construct $g_2$.
Define an attention layer $f_1:\fpq^{t_d\times n}\to\fpq^{t_d\times n}$ as 
\begin{align*}
f_1(Z)&=Z\oplus\left(W^O\otimes((W^V\otimes Z)\otimes\sigma((\zerov_{t_m,t_d}\otimes Z)^\top(\zerov_{t_m,t_d}\otimes Z)))\right)\\
&=Z\oplus\left(W^O\otimes((W^V\otimes Z)\otimes(\alpha\onev_{n,n}))\right)
\end{align*}
for some $W^V\in\fpq^{t_m\times t_d}$ and $W^O\in\fpq^{t_d\times t_m}$.
Specifically, we choose $W^V,W^O$ so that $W^V\otimes Z=Z_{t_d-\gamma+1:}$, $W^O_{:\din}=\zerov_{\din,t_m}$, $W^O_{\din+1:\din+\gamma}=I_{t_m}$, and $W^O_{\din+\gamma+1:}=\zerov_{t_d-\din-\gamma,t_m}$ where $I_{t_m}$ denotes the $t_m\times t_m$ identity matrix.
Then, by the definitions of $g_1,f_1$, it holds that for each $Z\in\fpq^{t_d\times n}$ with $Z_{:\din}=[\bfx_1,\dots,\bfx_n]\in\fpq^{\din\times n}$ and $Z_{\din+1:}=\zerov_{t_d-\din,n}$, for each $j\in[\gamma]$,
\begin{itemize}
\item $(f_1\circ g_1(Z))_{:\din}=[\bfx_1,\dots,\bfx_n]$, 
\item $(f_1\circ g_1(Z))_{\din+j}=(\bigoplus_{i=1}^{k_j}1^+)\times\onev_n^\top$ where $k_j$ denotes the number of $i\in[n]$ such that $\bfx_i=\bfz_j$,
\item $(f_1\circ g_1(Z))_{\din+\gamma+1:t_d-\gamma}=\zerov_{t_d-2\gamma-\din,n}$ and $(f_1\circ g_1(Z))_{t_d-\gamma+1:}$ is non-negative.
\end{itemize}

Choose a feed-forward network $f_2:\fpq^{t_d\times n}\to\fpq^{t_d\times n}$ as $f_2(Z)=Z\oplus(W_2\otimes\rho(W_1\otimes Z))$ where $W_1\in\fpq^{t_r\times t_d}$ and $W_1\in\fpq^{t_d\times t_r}$ are matrices satisfying that for $\bfz=(z_1,\dots,z_{t_d})\in\fpq^{t_d}$ and $\bfy=(y_1,\dots,y_{t_r})\in\fpq^{t_r}$,
\begin{align*}
W_1\otimes\bfz=(z_{t_d-\gamma+1},\dots,z_{t_d},0,\dots,0),~~
W_2\otimes\bfy=(0,\dots,0,-y_1,\dots,-y_{\gamma}),
\end{align*}
Namely, $f_2$ erases the last $\gamma$ coordinates of an input, if they are finite and non-negative.
Let $g_2=f_2\circ f_1$ and $\phi=g_2\circ g_1\in\mcB^{1,t_m,t_r,t_d,n}$. Then, by the definitions of $g_1,g_2$, $\phi$ satisfies the desired properties. This completes the proof.

\subsection{Proof of \cref{lem:max-distinguish}}\label{sec:pflem:max-distinguish}
Let $x_k= \bigoplus_{i=1}^k1^+$ for all $k\in[3\times2^\mbit-1]\cup\{0\}$.
By direct computation, we have 
$x_0=0$, $x_1=1^+$, $x_2=2^+$, $x_3=3^{++}$, $x_4=4^+$.
Furthermore, for $k\in\{4,\dots,2^\mbit\}$, $x_k=k^+$. This can be shown via mathematical induction. The base case $k=4$ holds by our direct computation. 
Suppose that $x_{k-1}=(k-1)^+$ for some $k\in\{5,\dots,2^\mbit\}$ and a natural number $e$ with $2^e<k\le2^{e+1}$; here, $e\ge2$. If $k\ne2^{e+1}$, then $x_{k-1}+1^+=k+2^{e-p}+2^{-p}$, i.e., $x_{k}=x_{k-1}\oplus 1^+=k^+$.
If $k=2^{e+1}$, then $x_{k-1}+1^+=2^{e+1}+2^{e-p}+2^{-p}$, i.e., $x_{k}=x_{k-1}\oplus 1^+=(2^{e+1})^+=k^+$.

Since $x_{2^\mbit}=2^\mbit+1$, for $k\in\{2^\mbit+1,\dots,2^{\mbit+1}-1\}$, one can observe that $x_k=k+1$. 
Furthermore, for $k\in\{2^{\mbit+1},\dots,3\times2^{\mbit}-1\}$, we have $x_{k}=2^{\mbit+1}+(k-2^{\mbit+1}+1)\times2$ due to the rounding. Hence, $x_k$ can be written as follows:
\begin{align*}
x_k=\begin{cases}
0~&\text{if}~k=0,\\
1^+~&\text{if}~k=1,\\
2^+~&\text{if}~k=2,\\
3^{++}~&\text{if}~k=3,\\
k^{+}~&\text{if}~k\in\{4,\dots,2^\mbit\},\\
k+1~&\text{if}~k\in\{2^\mbit+1,\dots,2^{\mbit+1}-1\},\\
2^{\mbit+1}+(k-2^{\mbit+1}+1)\times2~&\text{if}~k\in\{2^{\mbit+1},\dots,3\times2^{\mbit}-1\}.
\end{cases}
\end{align*}
This implies that $x_k$ are all distinct for all $k\in[3\times2^\mbit-1]\cup\{0\}$ and $x_{3\times2^\mbit-1}=2^{\mbit+2}$.

\subsection{Proof of \cref{lem:same-sum0}}\label{sec:pflem:same-sum0}
The proof is trivial when $x\in\{\pm\infty,\nan\}$.
For $x\in\fpq$, the result follows from \cref{lem:same-sum1}.

\subsection{Proof of \cref{lem:same-sum2}}\label{sec:pflem:same-sum2}
Without loss of generality, we assume that $z,x$ are finite and $x>0$; if $x\in\{0,\pm\infty,\nan\}$ or $z\in\{\pm\infty,\nan\}$, then the result directly follows.
We further assume that $z<0$ since if $z\ge0$, the result reduces to \cref{lem:same-sum1}.
Choose $e_z\in\bbZ$ such that $z=-2^{e_z}$. %
We now consider all possible cases.

{\bf Case $e_z\le\emin+1$.}
There exists $k\le 2^{\mbit+1}$ such that $z\oplus\bigoplus_{i=1}^kx\ge0$ since $x\ge2^{\emin-\mbit}=\omega$.
By \cref{lem:same-sum1}, the result follows.

{\bf Case $e_z>\emin+1$ and $x\le2^{e_z-\mbit-2}$.} %
In this case, 
$x\le2^{e_z-\mbit-2}=\frac12(z^+-z).$
Since $z=-2^{e_z}$, this implies that $z\oplus\bigoplus_{i=1}^kx=z$ for all $k\in\bbN$ and the result follows.

{\bf Case $e_z>\emin+1$ and $x>2^{e_z-\mbit-2}$.} %
Since 
$x>2^{e_z-\mbit-2}=\frac12(z^+-z)$,
it holds that $z\oplus\bigoplus_{i=1}^{2^\mbit}x\ge-2^{e_z-1}$.
Furthermore, it holds that
\begin{align*}
z\oplus\bigoplus_{i=1}^{3\times2^\mbit}x%
\ge-2^{e_z-1}\oplus\bigoplus_{i=1}^{2^{\mbit+1}}x\ge-2^{e_z-1}\oplus\bigoplus_{i=1}^{2^{\mbit+1}}2^{e_z-p-2}=0.
\end{align*}
Hence, by \cref{lem:same-sum1}, the result follows. This completes the proof.

\newpage
\section{Constructions with smaller $m,r,d$}\label{sec:deep-narrow}
In this section, we sketch the idea for proving \cref{thm:diagonal-approx,thm:perm-equiv-approx} with floating-point transformer constructions using $m,h,r=1$ and $d=\din+\dout+c$ for some constant $c\in\bbN$. 
In particular, we will focus on \cref{thm:diagonal-approx} under $\mbit\ge3$ since the idea here naturally extends to \cref{thm:perm-equiv-approx} and $\mbit=2$ (see the case when $p=2$ in  \cref{sec:pflem:diagonal-attention}).

We now describe how to transform our construction in the proof of \cref{lem:token-univ-approx} (see \cref{sec:pflem:token-univ-approx}) to a floating-point transformer with $m=r=h=1$ and $d=\din+\dout+c$.
By \cref{lem:fnn-univ-approx}, for any $f^*:\fpq^{d_1}\to\fpq^{d_2}$,
there exist $e_0=d_1,e_1,e_2,e_3,e_4,e_5,e_6=d_2\in\bbN$ and affine transformations $\phi_{i-1}:\efpq^{e_{i-1}}\to\efpq^{e_i}$ for all $i\in[6]$ such that for each $\bfx\in\fpq^{e_0}$,
$$\phi_6\circ\rho\circ\phi_5\circ\rho\circ\phi_4\circ\rho\circ\phi_3\circ\rho\circ\phi_2\circ\rho\circ\phi_1(\bfx)=f^*(\bfx).$$
By \cref{lem:narrow-fnn},
this implies that there exists a transformer block $g:\fpq^{d\times n}\to\fpq^{d\times n}$ with $m,r,h=1$ and $d=\din+\dout+6$ such that
$$g(X)_i=\begin{cases}
X_i&\text{if}~i\le\din,\\
X_i\oplus[f(\bfx_1),\dots,f(\bfx_n)]&\text{if}~\din<i\le\din+\dout,\\
\zerov_n^\top&\text{otherwise}.
\end{cases}$$

We now briefly sketch our construction for \cref{thm:diagonal-approx}.
Let $f^*:\Delta_n\to\fpq^{\dout\times n}$ be the target function.
Let $[\bfz_{1,1},\dots,\bfz_{1,n}],\dots,[\bfz_{t,1},\dots,\bfz_{t,n}]$ be all elements in $\Delta_n$, i.e., $|\Delta_n|=t$.
For each $i\in[t]$, we will make a transformer block $\phi_i$ with $m=r=h=1$ and $d=\din+\dout+9$  satisfying the following: for each $Z\in\fpq^{d\times n}$ with $Z_{:\din}=[\bfx_1,\dots,\bfx_n]\in\Delta_n$ and $Z_{\din+\dout+1:}=\zerov_{d-\din-\dout,n}$, 
$\phi_i(Z)_j=Z_j$ for all $j\in[\din]\cup([d]\setminus[\din+\dout])$ and
\begin{align*}
\phi_i(Z)_j=\begin{cases}
f^*(Z)_j~&\text{if}~Z=[\bfx_{i,1},\dots,\bfx_{i,n}],\\
Z_j~&\text{if}~Z\ne[\bfx_{i,1},\dots,\bfx_{i,n}],
\end{cases}
\end{align*}
for all $j\in[\din+\dout]\setminus[\din]$.
Let $W_\text{in}\in\fpq^{d\times\din}$ and $W_\text{out}\in\fpq^{\dout\times d}$ be matrices
satisfying $(W_\text{in}\otimes X)_{:\din}=X$, $(W_\text{in}\otimes X)_{\din+1:}=\zerov_{d-\din,n}$, $W_\text{out}\otimes Z=Z_{\din+1:\din+\dout}$.
Then, one can observe that
$W_\text{out}\otimes(\phi_t\circ\cdots\circ\phi_1(W_\text{in}\otimes X))=f^*(X)$ for all $X\in\Delta_n$ and completes the proof.

We now describe how to construct $\phi_i$.
Let $[\bfx_1,\dots,\bfx_n]\in\fpq^{\din\times n}$ be the input.
In the beginning of $\phi_i$, we set flag $1$ for the last dimension of all tokens.
For each $1\le i_1,i_2,i_3\le n$, let $\phi_{i,i_1,i_2,i_3}$ be a transformer block that checks whether there exist $1\le j_1<j_2<j_3\le n$ such that $\{\bfx_{j_1},\bfx_{j_2}\}=\{\bfz_{i,i_1},\bfz_{i,i_2}\}$ and $\bfx_{j_3}=\bfz_{i,i_3}$.
If so and if the previous flag is $1$, return flag $1$ in the last dimension for all tokens and return flag $0$ otherwise.
One can observe that such $\phi_{i,i_1,i_2,i_3}$ only requires $m=h=1$: for each token $i$, in the $(d-1)$-th dimension, put $\beta'$ if $x_i\in\{\bfz_{i,i_1},\bfz_{i,i_2}\}$, put $\beta$ if $x_i=\bfz_{i,i_3}$, and put $0$ otherwise where $\beta,\beta'\in\fpq$ such that $\beta\otimes(1\oslash(\bigoplus_{j=1}^n1^+))=1^+$ and $\beta'\otimes(1\oslash(\bigoplus_{j=1}^n1^+))=1^{++}.$
Extract this dimension using $W^V\in\fpq^{1\times d}$
and set $W^K=W^Q$ as zero matrices.
Then, we can check whether there exist $1\le j_1<j_2<j_3\le n$ such that $\{\bfx_{j_1},\bfx_{j_2}\}=\{\bfz_{i,i_1},\bfz_{i,i_2}\}$ and $\bfx_{j_3}=\bfz_{i,i_3}$ since $1^+\oplus1^{++}\oplus1^{++}=3^{+++}$ but $1^{++}\oplus1^{++}\oplus1^{+}=3^{++}$ when $\mbit\ge3$. Here, we do not modify the first $\din+\dout$ dimensions. Then, $d=\din+\dout+9$ suffices: 
first $\din+\dout$ dimensions for storing data, $6$ dimensions for representing token-wise floating-point functions (see \cref{lem:narrow-fnn}), one dimension for storing the flag, one dimension for storing $\beta,\beta',0$, one dimension for storing the output of the attention layer.
Hence, $\phi_{i,i_1,i_2,i_3}$ uses $m,r,h=1$ and $d=\din+\dout+9$. 
Let $\psi_{i,1}$ be the composition of all $\psi_{i,i_1,i_2,i_3}$. Then the last dimension of the output of $\psi_{i,1}$ is one if and only if the input matrix is equal to $[\bfz_{i,1},\dots,\bfz_{i,n}]$.

Let $\psi_{i,2}:\fpq^{d}\to\fpq^{d}$ be a feed-forward network such that
\begin{align*}
\psi_{i,2}(\bfx)_{:\din}&=\bfx_{:\din},\\
\psi_{i,2}(\bfx)_{\din+1:\din+\dout}&=\begin{cases}
\bfx_{\din+1:\din+\dout}+f^*([\bfz_{i,1},\dots,\bfz_{i,n}])~&\text{if}~x_d=1,\\  \bfx_{\din+1:\din+\dout}~&\text{otherwise},  
\end{cases},\\
\psi_{i,1}(\bfx)_{\din+\dout+1:}&=\zerov_{d-\din-\dout}.
\end{align*}
We can transform such a network to a transformer block with $m,r,h=1$ and $d=\din+\dout+9$ by \cref{lem:narrow-fnn}. Then, $\phi_i=\psi_{i,2}\circ\psi_{i,1}$ satisfies the desired properties with $m,r,h=1$ and $d=\din+\dout+9$.
We note that a similar idea can also be applied to \cref{thm:perm-equiv-approx}.

\end{document}